\documentclass{article} 
\usepackage{iclr2027_conference,times}

\usepackage{amsmath,amssymb,amsthm,booktabs,multirow}
\usepackage{algorithm}
\usepackage{algorithmic}
\usepackage{graphicx}
\usepackage{float}
\usepackage{wrapfig}
\usepackage{hyperref}
\usepackage{url}

\newtheorem{proposition}{Proposition}
\newtheorem{definition}{Definition}

\newcommand{\reservoir}{\mathcal{R}}

\title{Beyond Class Marginals:\\
Bounding Rehearsal Gaps without Freezing Class Co-occurrence}

\author{Congren Dai$^{1}$ \qquad Nat Roongjirarat$^{2}$ \qquad Fei Ye$^{3}$\thanks{Corresponding author.}\\
$^{1}$Imperial College London \quad $^{2}$King's College London\\
$^{3}$University of Electronic Science and Technology of China}

\iclrfinalcopy
\begin{document}

\setcounter{footnote}{1}
\maketitle
\lhead{ }

\begin{abstract}
Class-balanced replay controls class frequency but does not determine the
interval between successive replay appearances of a class. We study this
interval, the \emph{rehearsal gap}, separately from the class marginal and
class co-occurrence, and introduce \emph{randomised-pass replay} (RPR),
which visits each resident class once per shuffled pass. For a fixed set of
$C$ resident classes and replay batch size $b\le C$, RPR preserves the balanced
time-averaged class marginal and bounds every
gap by $2\lceil C/b\rceil-1$; a churn-conditional bound applies while the
resident set changes. The scheduler uses no future class information and adds
no replay examples or forward passes. In a linear-head ER-ACE diagnostic,
joint absence from the incoming and replay batches produces a one-sided
classifier-bias gradient. Longer absence episodes are associated with larger
negative bias displacement, and removing the incoming-loss mask attenuates
the scheduling effect. In the primary ER-ACE experiments, RPR improves final
average accuracy by $0.72$--$1.67$ percentage points relative to independent
class-balanced retrieval under reservoir storage, with positive effects also
observed under balanced storage.
Pretrained ViTs show positive effects on the tested LT10 streams with small
replay batches, while matched larger-batch controls show no material effect.
Fixed-cycle and reused-pass controls change more than one temporal statistic,
so the experiments do not isolate rehearsal-gap length from all other forms
of temporal dependence. The accuracy effects depend on the learner and
operating regime.
\end{abstract}

\section{Introduction}

Online continual learning requires a model to acquire new knowledge while
retaining classes that become infrequent or disappear from the incoming
stream. Experience replay addresses this by revisiting a small memory of
past examples \citep{chaudhry2019tinyepisodicmemoriescontinual}. Under an
imbalanced stream, storage alone does not determine when retained examples
are replayed. Storage-side balancing
\citep{chrysakis2020cbrs,buzzega2020bagoftricks} controls which classes
occupy memory, and class-balanced retrieval controls their expected share
of a replay batch. Neither specifies \emph{when} a stored class returns.

Replay order determines the sequence of gradient updates even when aggregate
class exposure is unchanged. Equal average exposure can coexist with long
periods of absence, during which other classes continue to update the model.
We call the interval between successive replay appearances of a class its
\emph{rehearsal gap}. Table~\ref{tab:gaps} gives a fixed-memory example. On
a fixed balanced memory with $200$ classes and eight replay examples per
step, independent balanced retrieval and a shuffled-pass schedule have mean
gaps of $24.82$ and $25.00$ steps, and maximum gaps over $3{,}000$ steps of
$263$ and $49$, respectively.

We distinguish three properties of a replay sequence: the \emph{class
marginal} is how often a class appears; the \emph{rehearsal gap} is how long
it waits between appearances; and \emph{class co-occurrence} is which
classes share a batch (Figure~\ref{fig:rpr}a). On a fixed resident set larger
than the replay batch, independent class-balanced retrieval gives a geometric
rehearsal-gap distribution with unbounded support. Fixed class cycling
\citep{hickok2024watch} gives bounded gaps and
repeated pairings determined by the cycle. At a fixed replay budget, we
study schedules that bound resident-class rehearsal gaps while preserving the
fixed-set time-averaged marginal and resampling class co-occurrence across
passes, and evaluate their effect on learning.

\begin{figure}[t]
  \centering
  \includegraphics[width=\textwidth]{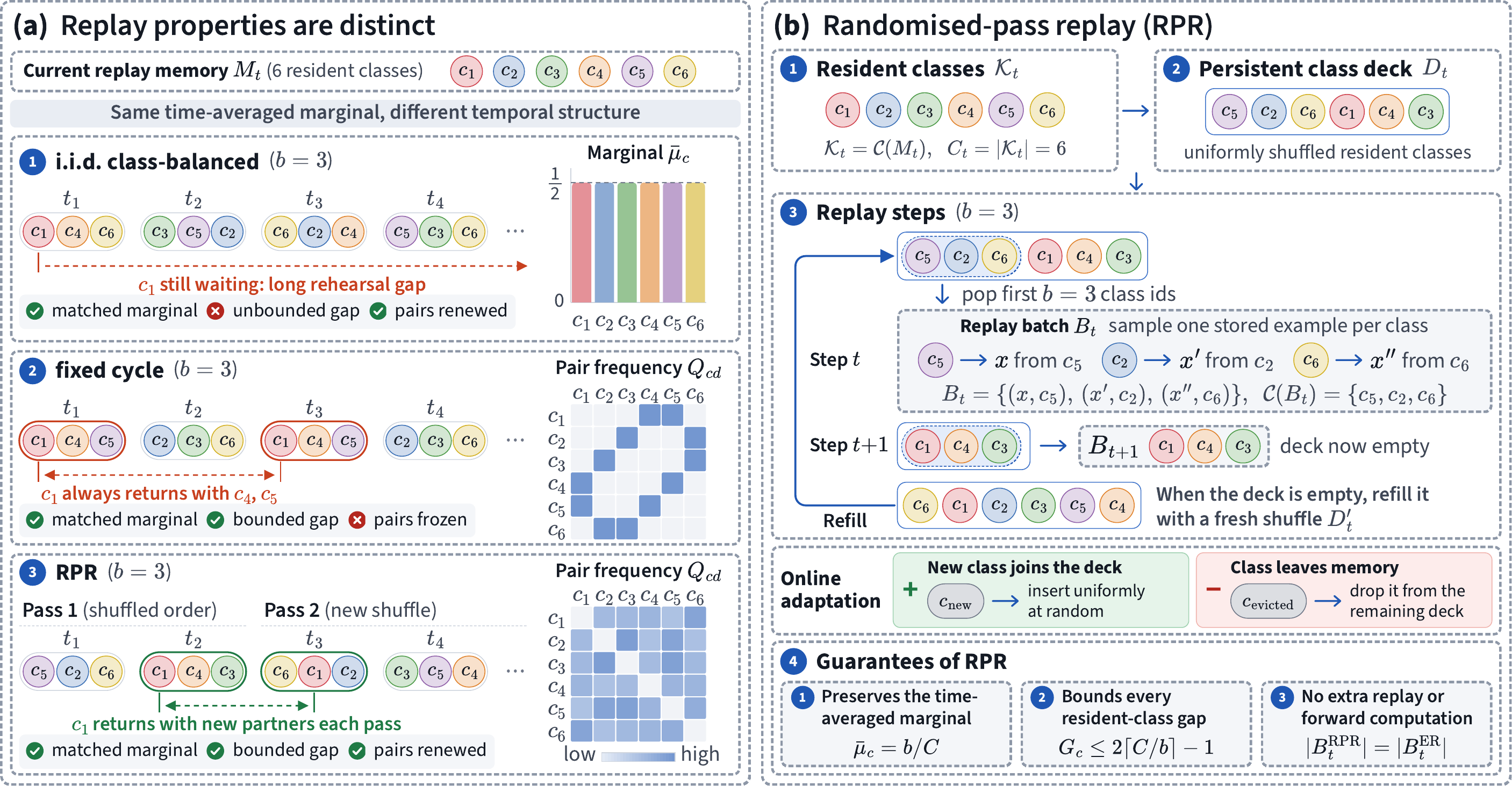}
  \caption{\textbf{Replay properties and randomised-pass replay (RPR).}
  (a) On a fixed resident set with replay batch $b\le C$, independent balanced
  retrieval, fixed cycling, and RPR have the same time-averaged class marginal
  $b/C$ but may differ in gap tails and class co-occurrence.
  (b) RPR keeps a persistent shuffled deck of resident classes, samples one
  stored example per selected class, and reshuffles when the deck empties.
  On a fixed resident set, it preserves $\bar{\mu}_c=b/C$ and bounds
  $G_c\le2\lceil C/b\rceil-1$. It adds no replay examples or forward passes.}
  \label{fig:rpr}
  \vspace{-20pt}
\end{figure}

Randomised-pass replay (RPR) maintains a persistent shuffled deck of
resident classes. Each class is visited once per pass, after which the deck
is reshuffled (Figure~\ref{fig:rpr}b). The deck is updated as classes enter
or leave memory. The scheduler requires no future class identities, total
class count, or task boundaries. On a fixed set of $C$ resident
classes with replay batch $b\le C$, its time-averaged marginal is balanced
and every gap is bounded by $2\lceil C/b\rceil-1$. A churn-conditional bound
covers changing memory; exact marginal equality applies only in the
fixed-set limit.

The scheduling guarantees do not imply an accuracy improvement. We analyse a
linear-head ER-ACE diagnostic \citep{caccia2022newinsights} to relate replay
gaps to parameter updates. The incoming-loss mask removes a seen but absent
class from incoming competition, while replay examples of other classes
continue to contribute a positive gradient to its classifier bias. During
joint absence, the class-specific bias receives gradients of the same sign
at every step while the class remains resident. We evaluate this mechanism
using measured bias displacement and mask interventions. The primary accuracy study uses a
bias-free cosine head. The bias-drift analysis therefore applies directly to
the linear-head diagnostic and does not by itself explain the primary
cosine-head results. We compare fixed cycling, reused permutations, storage
policies, and replay budgets to separate scheduling effects from storage
balance. Pretrained-backbone controls evaluate the same retrieval
intervention outside the primary ER-ACE configuration.

The shuffled-pass primitive is standard; our contributions concern its use
in replay.
\textbf{1) Replay-schedule decomposition.} We distinguish class marginal
frequency, rehearsal gaps, and class co-occurrence as separate properties of
a replay sequence.
\textbf{2) Scheduling guarantees.} For RPR, we derive the fixed-set
time-averaged marginal, a worst-case resident-class gap bound, and a
churn-conditional bound. We also give a lower bound on the maximum gap of any
recurrent schedule and, under the stated conditions, the fixed-partition
structure induced by attaining it exactly.
\textbf{3) Mechanistic and empirical evaluation.} We measure gap-conditioned
bias displacement in ER-ACE and evaluate retrieval order across storage
policies, replay losses, replay budgets, and pretrained backbones.
The temporal controls change more than one schedule statistic, so the
experiments do not isolate rehearsal-gap length from all other forms of
temporal dependence.

\section{Related Work}

\paragraph{Storage and retrieval.}
Reservoir sampling \citep{randomsamplingwithareservoir} retains a
stream-uniform sample and is standard in online continual learning
\citep{chaudhry2019tinyepisodicmemoriescontinual}. CBRS
\citep{chrysakis2020cbrs} and Balanced Reservoir Sampling (BRS)
\citep{buzzega2020bagoftricks} instead favour balanced occupancy; GSS
diversifies stored gradients \citep{aljundi2019gradient}, and InfoRS gates
reservoir eligibility by information \citep{sun2022infors}. These methods
modify memory contents rather than the temporal order of class retrieval.
Among retrieval methods, MIR selects examples whose loss increases after a
virtual incoming update \citep{aljundi2019mir}, and ASER uses a
Shapley-value proxy \citep{shim2021online}. Both optimise per-example
criteria rather than the class-level rehearsal-gap statistics studied here.
CeCR samples classes without replacement within a joint retrieval, storage,
and loss design \citep{sun2024cecr}. We compare schedules at fixed storage,
loss, and replay budget; exact time-averaged marginal matching applies
to the fixed-resident-set analysis.

\paragraph{Cyclic scheduling and reshuffling.}
The Balanced primitive of \citet{hickok2024watch} continues through classes
in fixed class-id order; their separate deduplication schedules limit
repeated examples. We include this class-cycle primitive as a baseline. It
gives bounded class gaps and, on a fixed resident set, a repeated class
partition when $b\mid C$.
Task-level scheduling instead learns which
previous tasks to replay \citep{klasson2023learn}, whereas RPR orders
resident classes without task boundaries. Random reshuffling is well
established for finite-sum optimisation
\citep{gurbuzbalaban2021random,haochen2019random,mishchenko2022proximal}.
RPR applies it to a changing set of resident \emph{class identifiers}
instead of a static set of objective components. Our analysis concerns replay
gaps, class co-occurrence, and their interaction with learning; it does not
use finite-sum convergence results.

\paragraph{Replay losses.}
ER-ACE masks seen-but-absent classes from the incoming loss
\citep{caccia2022newinsights}. DER regresses stored logits, and DER++ adds
supervised replay cross-entropy \citep{buzzega2020dark}. Contrastive and
prototype-based methods couple replay examples through batch-level
objectives \citep{mai2021scr,guo2022ocm,wei2023online}.
Definition~\ref{def:blind} characterises replay terms in which each example's
gradient contribution is independent of the other examples in the batch;
Section~\ref{sec:lawresults} relates this property to full training trajectories.

\section{Setting and Notation}
\label{sec:setting}

A task-free online learner processes a stream
$S=((x_i,y_i))_{i=1}^{N}$ once, without task-boundary signals. At step $t$,
it receives an incoming minibatch $A_t$, retrieves $B_t$ from a memory
$M_t$ of at most $m$ previously observed examples, and takes one gradient step on
\begin{equation}
  \mathcal{L}(\theta_t)
  =\ell^{\mathrm{on}}(\theta_t;A_t)+\ell^{\mathrm{re}}(\theta_t;B_t).
  \label{eq:loss}
\end{equation}
The requested replay budget is $n_t=\min(b,|M_t|)$ and the realised size obeys
$|B_t|\le n_t$ (Appendix~\ref{app:algorithm}); an empty memory contributes no
replay loss. Write $\mathcal{C}(B)=\{y:(x,y)\in B\}$ and
$\mathcal{S}_t=\bigcup_{s\le t}\mathcal{C}(A_s)$ for the classes in a batch
and those seen so far. The scheduler uses only the current resident set
\begin{equation}
  \mathcal{K}_t=\mathcal{C}(M_t),\qquad C_t=|\mathcal{K}_t|.
  \label{eq:residentset}
\end{equation}
It requires neither the identities nor the number of future classes: a class
enters $\mathcal{K}_t$ only after observation and storage. We use $C$ for a
fixed resident count in the analysis, not an advance input to RPR.
Storage determines which examples remain in $M_t$; retrieval determines
which return in $B_t$. Section~\ref{sec:seams} varies these rules independently.

\section{Randomised Pass Replay}
\label{sec:method}
\vspace{-3pt}
\subsection{Replay Marginals, Gaps, and Co-occurrence}
\vspace{-3pt}
For a retrieval policy $\pi$, distinguish the history-conditional inclusion
probability $\mu_c^{(t)}=\Pr[c\in\mathcal{C}(B_t)\mid\mathcal{F}_t]$, conditional
on the history available before retrieval at step $t$, from the time-averaged marginal
\begin{equation}
  \bar\mu_c(\pi)=\lim_{N\to\infty}\frac{1}{N}
  \sum_{t=1}^{N}\Pr[c\in\mathcal{C}(B_t)],
  \label{eq:timeavg}
\end{equation}
when this limit exists on a fixed resident set. Finite online runs use
empirical visit frequencies. I.i.d.\ class-balanced retrieval independently
selects $b$ distinct classes per step, giving both marginals $b/C$ for
$b\le C$. RPR does not preserve the history-conditional inclusion
probability, since classes already visited in the current pass are
temporarily unavailable. We use \emph{matched marginal} only for the
fixed-set time average in Equation~\eqref{eq:timeavg}. Under class arrivals,
evictions, and incomplete passes, finite-horizon empirical frequencies may
differ across retrieval policies. The online comparisons match storage,
loss, and replay size, but do not establish exact per-class frequency
equality under churn.

The empirical frequency over $T$ steps is
$\hat\mu_c=T^{-1}\sum_{t=1}^{T}\mathbf{1}[c\in\mathcal{C}(B_t)]$.
Equal replay size fixes the total example budget but does not fix each
$\hat\mu_c$, and a similar aggregate mean gap does not establish per-class
equality.

Let $\tau_c(k)$ denote the $k$-th replay appearance of class $c$. Its
\emph{rehearsal gap} and the empirical frequency of a distinct class pair are
\begin{equation}
  G_c(k)=\tau_c(k+1)-\tau_c(k),\qquad
  Q_{cd}^{(N)}=\frac{1}{N}\sum_{t=1}^{N}
  \mathbf{1}[\{c,d\}\subseteq\mathcal{C}(B_t)].
  \label{eq:gap}
\end{equation}
Neither statistic is determined by $\bar\mu_c$. For i.i.d.\ balanced
retrieval on a fixed set with $b<C$,
\begin{equation}
  G_c\sim\mathrm{Geometric}(p),\qquad
  \mathbb{E}[G_c]=C/b,\qquad
  \Pr[G_c>g]=(1-p)^g,\qquad p=b/C.
  \label{eq:geometric}
\end{equation}
For i.i.d.\ balanced retrieval, $\mathbb{E}[G_c]=C/b$, with unbounded gap
support when $b<C$. Fixed class cycling gives bounded rehearsal gaps and,
when $b\mid C$, repeated class pairings within a fixed partition
(Figure~\ref{fig:rpr}a).

\subsection{Randomised Pass Scheduler}
\label{sec:rpr}

RPR maintains a persistent deck of resident classes not yet visited in the
current pass. When $b\le C_t$, the scheduler removes $b$ class identifiers from
the deck and samples one stored example from each selected class. When the
deck empties, it is refilled with a uniform random permutation of
$\mathcal{K}_t$. If a pass boundary falls within a minibatch, classes
already visited in that step move to the end of the new permutation; this
prevents duplicates when $b\le C_t$. A newly resident class enters a
uniformly random position among unvisited entries, and an evicted class
leaves the deck. Insertions and evictions update the current deck without
restarting the pass. The scheduler requires no information about future
classes. If $b\ge C_t$, retrieval covers every resident class. Appendix~\ref{app:algorithm} specifies the
complete rule, including warm-up and within-class sampling.

\begin{proposition}[Preserved time-averaged marginal]
\label{prop:marginal}
On a fixed resident set of size $C$ with $b\le C$, RPR visits each class
once per pass. Hence $\bar\mu_c=b/C$ and the event-average gap is
$\mathbb{E}[G_c]=C/b$, as under i.i.d.\ balanced retrieval. The
history-conditional per-step marginals need not coincide.
\end{proposition}

\begin{proposition}[Bounded gap on a fixed resident set]
\label{prop:gap}
Under the same conditions, every class and visit satisfy
\begin{equation}
  G_c(k)\le 2\Big\lceil \frac{C}{b}\Big\rceil-1.
  \label{eq:bound}
\end{equation}
The bound is tight when $b$ divides $C$: a class can appear in the first
batch of one pass and the last batch of the next.
\end{proposition}

\begin{proof}
Let $q=\lceil C/b\rceil$, and let a pass start after $a$ consumed class slots.
Its last slot is $a+C$. Write $a+C=hb+r$, with $0\le r<b$.
If $r>0$, its final $r$ classes share a batch with the next pass and move
to the end of that next pass. A class outside this deferred set therefore
appears among the first $C-r$ slots of the next pass, no later than batch
$h+q$. Its preceding appearance is no earlier than batch
$\lfloor a/b\rfloor+1$, so its gap is at most
$h+q-\lfloor a/b\rfloor-1\le2q-1$.
A deferred class instead appears in batch $h+1$ and returns no later than
$h+\lceil(C+r)/b\rceil$, giving a gap at most $q\le2q-1$.
When $r=0$, the first case covers every class. Every pass remains a
permutation and no step repeats a class for $b\le C$; consequently each
complete pass contributes one class visit and $C$ slots, which also proves
Proposition~\ref{prop:marginal}.
\end{proof}

The bound does not assert tightness at other ratios. Online storage
requires a separate statement.

\begin{proposition}[Churn-conditional bound]
\label{prop:churn}
Suppose $c$ remains resident between consecutive replay appearances.
Let $C_{\max}$ be the largest resident count during this interval and $I$
the number of class insertions into the deck. Then
\begin{equation}
  G_c(k)\le\Big\lceil\frac{2C_{\max}-1+I}{b}\Big\rceil.
  \label{eq:churnbound}
\end{equation}
\end{proposition}

This conservative slot-count bound need not equal Equation~\eqref{eq:bound}
even when $I=0$. Both require residency: no retrieval rule bounds the latency
of a class that storage removes entirely. We distinguish
\emph{resident gaps}, whose endpoints and intervening steps retain the class
in memory, from \emph{wall gaps}, which also include periods of eviction.
Section~\ref{sec:dyngaps} tests the churn bound inside the online runs.

\subsection{Gap Bounds and Co-occurrence}
\label{sec:theory}

Any recurrent schedule has maximum rehearsal gap at least
$\lceil C/b\rceil$, since fewer than $\lceil C/b\rceil$ steps cannot cover
all $C$ classes. Proposition~\ref{prop:gap} gives
$G_c(k)\le 2\lceil C/b\rceil-1$. When $b\mid C$, attaining $G_c(k)=C/b$ for
every class and visit forces a fixed partition into batches: each class
always shares its batch with the same $b-1$ classes. This statement requires
exact equality and does not extend to schedules whose gaps only approach the
lower bound.

Each RPR pass induces a new uniform partition when $b\mid C$, with expected
pair frequency
\begin{equation}
  \mathbb{E}[Q_{cd}]=\frac{b}{C}\frac{b-1}{C-1}\quad(c\ne d).
  \label{eq:paircoverage}
\end{equation}
A fixed partition restricts each class to the $b-1$ classes in its block.
Under this divisibility condition, every distinct class pair has the co-occurrence probability in
Equation~\eqref{eq:paircoverage}. This co-occurrence result alone does not
imply an improvement in predictive accuracy.

Over $n$ rehearsal gaps per class, a union bound gives an i.i.d.\ maximum-gap scale
of $O(\log(nC)/\log(1/(1-p)))$ for $p=b/C<1$. A single-class lower bound
has the same logarithmic dependence. The RPR bound is independent of $n$.

When $b\mid C$ and $C/b$ increases, the ratio between the i.i.d.\ and RPR
gap variances approaches $6$, corresponding to a standard-deviation ratio
of $\sqrt{6}\approx2.449$. The four static configurations in
Table~\ref{tab:gaps} give standard-deviation ratios $2.41$, $2.42$, $2.38$,
and $2.43$. Appendix~\ref{app:theory} gives the derivations and the convex
gap-cost result. These results characterise the replay schedule and do not
imply an accuracy improvement; the learning consequences depend on the loss
and the evolving parameters.

\subsection{Interaction with ER-ACE}
\label{sec:law}

We first consider replay losses that decompose over examples.

\begin{definition}[Composition-blind replay]
\label{def:blind}
A replay term is \emph{composition-blind} if it decomposes over examples,
\begin{equation}
  \ell^{\mathrm{re}}(\theta;B)\;=\;\sum_{(x,y)\in B}\phi(\theta;x,y),
  \label{eq:blind}
\end{equation}
with per-example $\phi$ independent of the other examples, including their
labels and any joint forward computation.
\end{definition}

If Equation~\eqref{eq:blind} holds then
$\nabla_\theta\ell^{\mathrm{re}}(\theta;B)=\sum_{(x,y)\in B}\nabla_\theta\phi(\theta;x,y)$,
so one example's contribution is unaffected by its batchmates' labels, and
\begin{equation}
  \mathbb{E}_{B\sim\pi}\big[\nabla_\theta\ell^{\mathrm{re}}(\theta;B)\big]
  \;=\;
  \sum_{(x,y)\in M_t}\!\!\Pr_\pi\big[(x,y)\in B\big]\,
  \nabla_\theta\phi(\theta;x,y),
  \label{eq:blindexp}
\end{equation}
which depends only on the per-\emph{example} inclusion probabilities. At
fixed $(\theta,M_t)$, two policies that match those probabilities have the
same one-step expected replay gradient. Their full training trajectories
need not coincide: class-balanced and uniform retrieval need not match
per-example probabilities in a finite, uneven store, and temporal
dependence changes the future parameters at which gradients are evaluated.

Dark Experience Replay \citep{buzzega2020dark} satisfies
Definition~\ref{def:blind} for a batch-independent forward map. Its replay term is
$\alpha\,\|f_\theta(x)-z\|_2^2$ summed over the batch, a per-example
regression on stored logits that does not depend on the class label.

ER-ACE applies standard cross-entropy to replay examples and masks previously
observed classes that are absent from $A_t$ in the incoming-example softmax
\citep{caccia2022newinsights}. For a seen class $c$ absent from both batches,
the incoming term supplies no class-$c$ logit gradient, while replay supplies
only positive softmax contributions. To obtain an explicit parameter-level
statement, consider a nonempty replay batch and the linear-head classifier
$z_c(x)=w_c^\top h_\theta(x)+\beta_c$,
\begin{equation}
  \frac{\partial\mathcal{L}_t}{\partial\beta_c}
  =\sum_{x\in B_t}p_c(x)>0
  \qquad
  \bigl(c\in\mathcal{S}_t\setminus
  [\mathcal{C}(A_t)\cup\mathcal{C}(B_t)]\bigr).
  \label{eq:onesided}
\end{equation}
Here $p_c(x)$ is the replay softmax probability for class $c$ at input $x$.
A replay gap of $g$ steps has $g-1$ interior steps without rehearsal.
If the class also remains absent from the stream, each contributes
one-sided bias pressure; plain gradient descent decreases $\beta_c$ at
each such step. Gap length alone does not fix the magnitude, since the
probabilities change during training.

Equation~\eqref{eq:onesided} characterises the bias gradient and does not
imply a decrease in every class-$c$ logit. Shared representations couple
outputs through the empirical neural tangent kernel \citep{jacot2018neural};
momentum and weight decay introduce additional terms described in
Appendix~\ref{app:jacobian}. The experiments measure realised bias
displacement and mask--schedule interactions directly. The proposed mechanism
predicts sensitivity to long replay gaps at fixed mean rehearsal frequency.
The primary cosine classifier has no bias parameter, so
Equation~\eqref{eq:onesided} applies only to the linear-head diagnostic and
does not directly explain the gains of the primary classifier.

\section{Experiments}
\label{sec:experiments}

The experiments evaluate replay order under class-balanced retrieval,
rehearsal gaps in online runs, the linear-head ER-ACE mechanism,
and sensitivity to replay budget, storage policy, loss, and backbone.

\subsection{Experimental Setup}
\label{sec:protocol}

\paragraph{Benchmarks and streams.}
The evaluation uses class-incremental streams from CIFAR-10 and CIFAR-100
\citep{krizhevsky2009learning}, Tiny-ImageNet \citep{le2015tiny}, and
ImageNet-R \citep{hendrycks2021many}. The CIFAR-10 control uses a replay batch
that accommodates every class once all classes are resident. The main experiments
use CIFAR-100 and Tiny-ImageNet with balanced streams or exponential
long-tailed streams at imbalance factors $\rho\in\{10,100\}$. Within each
task, a seeded random permutation assigns class rank $r$, and the retained
example count is
$n_r=\max\{1,\operatorname{round}(n_{\max}\rho^{-r/(C_{\mathrm{task}}-1)})\}$.
This is a within-task adaptation of the standard artificial long-tailed CIFAR
construction \citep{cui2019classbalanced}. ImageNet-R is used for the
pretrained-backbone evaluation.

\paragraph{Methods and baselines.}
ER-ACE with a randomly initialised ResNet-18 \citep{DeepRes} serves as the
primary learner. We additionally examine DER and DER++
\citep{buzzega2020dark}, and Online Continual
Learning through Mutual Information Maximization (OCM)
\citep{guo2022ocm}. These additional losses test whether retrieval effects
depend on the replay objective. The pretrained-backbone study uses a
224-pixel ViT-B/16
\citep{dosovitskiy2021imageworth16x16words}, pretrained on ImageNet-21k and
AugReg fine-tuned on ImageNet-1k \citep{steiner2022train}. We load the exact
\path{vit_base_patch16_224.augreg2_in21k_ft_in1k} checkpoint from timm
\citep{wightman2019timm}. Retrieval policies include
uniform sampling, i.i.d.\ class-balanced sampling, fixed class cycling, MIR,
and RPR. Reservoir sampling and Balanced Reservoir Sampling provide the two
storage policies in the storage--retrieval study.

\paragraph{Evaluation.}
We report final task-mean test accuracy and mean paired differences in percentage
points. Comparisons use ten paired random seeds and two-sided $95\%$
Student-$t$ confidence intervals. Full paired-test results include counts of
positive paired differences. We define an effect as \emph{material} when
its relative magnitude exceeds $2\%$ and its confidence interval excludes
zero.
Separate validation runs select the tested method and regimes without entering
test comparisons.

\paragraph{Implementation details.}
All experiments process each stream once. The primary ResNet-18 ER-ACE
setting uses SGD with learning rate $0.03$, no momentum or weight decay,
incoming batch size $32$, and replay batch size $8$. CIFAR-100 and Tiny-ImageNet use
ten class-incremental tasks of ten and twenty classes, respectively, with
random crops, horizontal flips, and dataset-specific normalisation. Task
boundaries organise the stream and evaluation, but do not enter the replay
rule. Buffer capacities range from $200$ to $5{,}120$ where indicated.
The pretrained ViT setting uses a replay batch of $32$ examples unless the
replay-budget ablation states otherwise. Accuracy is
evaluated with an exponential moving average of the weights over a
$1{,}024$-sample horizon \citep{soutif2023temporal}, held fixed within every
comparison. Section~\ref{sec:emarobust} reports the corresponding comparison
without the EMA readout. Appendix~\ref{app:algorithm} gives the RPR
pseudocode, and the remaining appendices provide complete paired statistics
and additional controls.

\paragraph{Classifier and comparison scope.}
The primary ER-ACE accuracy study uses a normalised cosine classifier with
scale $10$ and no bias. The bias-drift, mask--schedule, renewal, and buffer
diagnostics use an ordinary linear classifier with bias, with the same
configuration across each block. This permits bias-displacement measurements
and limits direct mechanistic attribution to that head. The from-scratch
accuracies characterise a one-pass, small-replay-budget regime; they are not
a reproduction of published ER-ACE scores under other training protocols.
Comparisons estimate retrieval effects within each protocol, including the
pretrained setting, and do not establish generality across training regimes.

\subsection{Results}
\label{sec:results}
\begingroup
\setlength{\intextsep}{4pt}
\paragraph{Retrieval order and storage balance.}
\label{sec:seams}

\begin{table}[t!]
\caption{Storage and retrieval contrasts in final accuracy. Entries give
differences in percentage points (relative differences in parentheses), with
$95\%$ paired confidence intervals where shown. Storage compares BRS with
reservoir at uniform retrieval; draw compares i.i.d.\ class-balanced with
uniform retrieval at reservoir storage; both compares BRS with RPR against
reservoir with uniform retrieval. RPR comparisons specify the storage rule
($\mathcal{R}$: reservoir). For RPR--fixed contrasts, stars denote
Holm-adjusted $p<.05/.01/.001$. Full accuracies and paired tests appear in
Appendices~\ref{app:accuracies} and~\ref{app:statistics}.}
\label{tab:contrasts}
\centering
\footnotesize
\setlength{\tabcolsep}{3pt}
\resizebox{\linewidth}{!}{%
\begin{tabular}{l rrrr}
\toprule
Contrast & \shortstack{CIFAR-100\\Balanced} & \shortstack{CIFAR-100\\LT10} & \shortstack{CIFAR-100\\LT100} & \shortstack{Tiny-ImageNet\\LT10} \\
\midrule
storage
 & \shortstack{$+0.15$ \scriptsize($+0.8\%$)\\ \scriptsize$[-0.44, 0.73]$}
 & \shortstack{$+1.86$ \scriptsize($+19.7\%$)\\ \scriptsize$[1.33, 2.40]$}
 & \shortstack{$+0.88$ \scriptsize($+13.2\%$)\\ \scriptsize$[0.48, 1.28]$}
 & \shortstack{$+0.99$ \scriptsize($+14.0\%$)\\ \scriptsize$[0.48, 1.51]$} \\
draw
 & \shortstack{$+0.67$ \scriptsize($+3.6\%$)\\ \scriptsize$[0.10, 1.23]$}
 & \shortstack{$+1.80$ \scriptsize($+19.0\%$)\\ \scriptsize$[1.30, 2.30]$}
 & \shortstack{$+1.31$ \scriptsize($+19.6\%$)\\ \scriptsize$[0.54, 2.08]$}
 & \shortstack{$+0.66$ \scriptsize($+9.4\%$)\\ \scriptsize$[0.23, 1.10]$} \\
\midrule
RPR--i.i.d. $\mid$ \textsc{brs}
 & \shortstack{$+1.22$ \scriptsize($+6.3\%$)\\ \scriptsize$[0.43, 2.02]$}
 & \shortstack{$+1.38$ \scriptsize($+12.1\%$)\\ \scriptsize$[1.02, 1.74]$}
 & \shortstack{$+0.58$ \scriptsize($+7.0\%$)\\ \scriptsize$[0.25, 0.91]$}
 & \shortstack{$+1.77$ \scriptsize($+22.4\%$)\\ \scriptsize$[1.37, 2.17]$} \\
RPR--fixed $\mid$ \textsc{brs}
 & \shortstack{$+1.91^{***}$ \scriptsize($+10.3\%$)\\ \scriptsize$[1.48, 2.35]$}
 & \shortstack{$+1.03^{**}$ \scriptsize($+8.8\%$)\\ \scriptsize$[0.56, 1.51]$}
 & \shortstack{$+0.14$ \scriptsize($+1.6\%$)\\ \scriptsize$[-0.14, 0.43]$}
 & \shortstack{$+1.13^{***}$ \scriptsize($+13.2\%$)\\ \scriptsize$[0.84, 1.42]$} \\
RPR--i.i.d. $\mid \reservoir$
 & \shortstack{$+1.19$ \scriptsize($+6.1\%$)\\ \scriptsize$[0.75, 1.63]$}
 & \shortstack{$+1.40$ \scriptsize($+12.5\%$)\\ \scriptsize$[1.06, 1.74]$}
 & \shortstack{$+0.72$ \scriptsize($+9.0\%$)\\ \scriptsize$[0.29, 1.16]$}
 & \shortstack{$+1.67$ \scriptsize($+21.5\%$)\\ \scriptsize$[1.28, 2.06]$} \\
RPR--fixed $\mid \reservoir$
 & \shortstack{$+2.06^{***}$ \scriptsize($+11.1\%$)\\ \scriptsize$[1.53, 2.58]$}
 & \shortstack{$+0.59^{*}$ \scriptsize($+4.9\%$)\\ \scriptsize$[0.13, 1.05]$}
 & \shortstack{$+0.47$ \scriptsize($+5.6\%$)\\ \scriptsize$[-0.04, 0.97]$}
 & \shortstack{$+0.97^{**}$ \scriptsize($+11.5\%$)\\ \scriptsize$[0.46, 1.49]$} \\
\midrule
both
 & $+1.80$ \scriptsize($+9.6\%$)
 & $+3.38$ \scriptsize($+35.8\%$)
 & $+2.26$ \scriptsize($+33.8\%$)
 & $+2.60$ \scriptsize($+36.6\%$) \\
\bottomrule
\end{tabular}}
\end{table}

The primary comparison replaces independent class-balanced retrieval with RPR
at fixed storage, learner, and replay budget.
Table~\ref{tab:contrasts} reports gains of $0.72$--$1.67$ percentage points
under reservoir storage. These online comparisons test the scheduling
policy as a whole; Proposition~\ref{prop:marginal} does not guarantee
identical finite-run replay frequencies in their evolving memories.
Crossing storage and retrieval rules evaluates scheduling at different levels
of storage imbalance. On CIFAR-100 LT10, BRS reduces the
occupancy spread from $15.4\times$ to $1.02\times$. Within BRS, RPR increases accuracy
by $1.38$ points over independent balanced retrieval, with all ten paired
differences positive. The mean RPR effect is positive under both storage
rules in every tested stream.

Fixed cycling provides a bounded-gap comparator without permutation
renewal. RPR has higher mean accuracy in all eight storage--stream combinations, but
only six intervals exclude zero; both CIFAR-100 LT100 contrasts remain
inconclusive. All eight RPR--i.i.d.\ contrasts and six RPR--fixed contrasts
remain significant after the within-family Holm adjustment
(Appendix~\ref{app:statistics}).

\paragraph{Rehearsal gaps under changing memory.}
\label{sec:dyngaps}

Telemetry evaluates the churn-conditional bound as memory changes.
Across CIFAR-100 LT10, Tiny-ImageNet LT10, and Tiny-ImageNet LT100, RPR's mean per-run
resident maxima are $23.9$, $48.1$, and $46.3$ steps, with no violation of
the corresponding churn-conditional bound in any run
(Appendix~\ref{app:dyngaps}). Independent balanced draws have similar
aggregate mean gaps, but mean per-run maxima of $81$--$176$ steps. Similar means do not
establish matched per-class frequencies.

\begin{figure}[t]
\centering
\includegraphics[width=\linewidth]{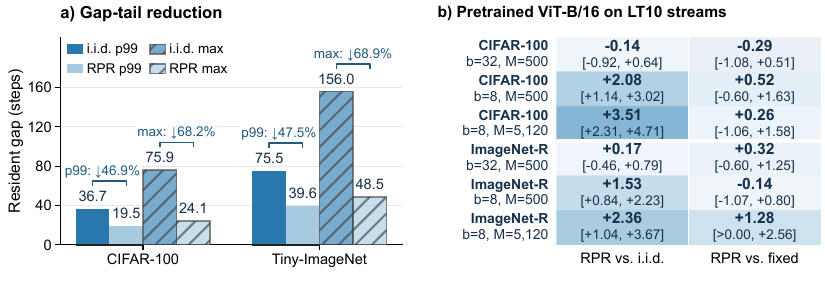}
\caption{Temporal coverage and pretrained ViT evaluation. \textbf{a}, Linear-head ER-ACE
resident-gap 99th percentiles and maxima on LT10 streams (ten-seed means); arrows show relative reductions.
\textbf{b}, RPR minus indicated comparator (percentage points; $95\%$ paired
confidence intervals, ten seeds).
Row labels give replay batch $b$ and buffer capacity $M$;
colours are centred at zero. Table~\ref{tab:gap-opportunity} gives baseline gaps.}
\label{fig:mechanism-summary}
\end{figure}

Tiny-ImageNet LT100 has $7.4$ class deactivations per run. Its maximum wall
gap exceeds the resident maximum because it includes eviction periods,
during which the class is unavailable to any retrieval schedule. Relative to
independent balanced retrieval, mean accuracy differences in the from-scratch settings are positive in all
stream-frequency thirds, while macro per-class forgetting falls by
$1.74$--$2.75$ points (Appendix~\ref{app:groups}).

\paragraph{Linear-head bias displacement and mask interventions.}
\label{sec:gapdrift}

Under masked linear-head ER-ACE, every completed joint-absence episode in
the telemetry has negative classifier-bias displacement. Mean displacement
magnitude increases across the duration bins. For CIFAR-100 under independent
balanced retrieval, the mean changes from
$-0.00191$ at $1$--$4$ steps to $-0.00926$ at $\geq33$ steps.
Tiny-ImageNet shows the same ordering (Appendix~\ref{app:review-followup}).
These measured parameter changes are consistent with the accumulated bias
pressure in Equation~\eqref{eq:onesided}. The gradient-sign result alone
does not establish an accuracy effect.

\begin{wraptable}{r}{0.48\linewidth}
\setlength{\abovecaptionskip}{0pt}
\setlength{\belowcaptionskip}{3pt}
\centering
\caption{Mask--schedule interaction in linear-head ER-ACE on LT10 streams.
RPR minus i.i.d.\ accuracy (percentage points; $95\%$ paired CI);
interaction: masked minus unmasked effect.}
\label{tab:mask-interaction}
\resizebox{\linewidth}{!}{%
\begin{tabular}{lcc}
\toprule
 & CIFAR-100 & Tiny-ImageNet \\
\midrule
Masked      & $+1.23\;[+0.53,+1.92]$ & $+1.16\;[+0.73,+1.59]$ \\
No mask     & $+0.30\;[-0.31,+0.91]$ & $+0.09\;[-0.33,+0.52]$ \\
Interaction & $+0.93\;[+0.37,+1.49]$ & $+1.06\;[+0.68,+1.45]$ \\
\bottomrule
\end{tabular}%
}
\end{wraptable}

In the linear-head diagnostics, RPR reduces the resident-gap $99$th percentile by
about $47\%$ on both datasets (Figure~\ref{fig:mechanism-summary}a).
With the incoming mask, RPR gains $1.23$ and $1.16$ points; without it, both
intervals include zero. The paired mask--schedule interactions are $0.93$ and
$1.06$ points with intervals above zero (Table~\ref{tab:mask-interaction}),
supporting attenuation after mask removal.

Reusing a shuffled pass for $4$ or $16$ traversals reduces pair coverage
and entropy, and fixed cycling reduces both further. All these bounded
schedules have higher mean accuracy than independent balanced retrieval, but
accuracy is non-monotone in pair diversity, and fixed cycling has the shortest
gaps without the largest gain. These controls do not isolate the gap tail from
other temporal dependencies.

\paragraph{Dependence on replay loss, storage, budget, and backbone.}
\label{sec:lawresults}
\label{sec:failed}

Loss interventions evaluate the scope of the ER-ACE mechanism. Removing
its mask, or adding it to DER++, leaves the balanced-draw versus uniform
contrast inconclusive (Appendix~\ref{app:law}). RPR increases accuracy by $1.02$ and $0.87$ points
over independent balanced retrieval in masked DER++ on CIFAR-100 and Tiny-ImageNet, respectively.
DER, DER++/Refresh, and OCM yield inconclusive or dataset-dependent results;
these comparisons do not support a common trajectory-level effect based on
the loss's one-step form.

\label{sec:imbalance}
\label{sec:vitscope}
\begin{wraptable}[19]{r}{0.49\linewidth}
\setlength{\abovecaptionskip}{0pt}
\setlength{\belowcaptionskip}{3pt}
\centering
\caption{Pretrained ViT-B/16. Gap: i.i.d.\ resident-gap $99$th percentile.
$\Delta$: RPR minus i.i.d.\ final task-mean accuracy (percentage points;
$95\%$ paired CI).}
\label{tab:gap-opportunity}
\footnotesize
\setlength{\tabcolsep}{3pt}
\begin{tabular*}{\linewidth}{@{\extracolsep{\fill}}lrr}
\toprule
Setting $(b,M)$ & Gap & $\Delta$ accuracy \\
\midrule
\multicolumn{3}{l}{\textbf{CIFAR-100}}\\
Bal.\ $(32,500)$ & 9.00 & $-1.09\;[-2.03,-0.15]$ \\
Bal.\ $(16,500)$ & 18.30 & $+0.30\;[-1.14,+1.73]$ \\
Bal.\ $(8,500)$ & 37.60 & $+0.66\;[-0.27,+1.60]$ \\
LT10 $(32,500)$ & 8.00 & $-0.14\;[-0.92,+0.64]$ \\
LT10 $(8,500)$ & 34.20 & $+2.08\;[+1.14,+3.02]$ \\
LT10 $(8,5120)$ & 36.10 & $+3.51\;[+2.31,+4.71]$ \\
\addlinespace
\multicolumn{3}{l}{\textbf{ImageNet-R}}\\
Bal.\ $(32,500)$ & 15.80 & $+0.22\;[-0.76,+1.20]$ \\
Bal.\ $(16,500)$ & 31.70 & $+0.48\;[-0.27,+1.23]$ \\
Bal.\ $(8,500)$ & 63.90 & $+0.85\;[-0.33,+2.02]$ \\
LT10 $(32,500)$ & 15.60 & $+0.17\;[-0.46,+0.79]$ \\
LT10 $(8,500)$ & 61.40 & $+1.53\;[+0.84,+2.23]$ \\
LT10 $(8,5120)$ & 67.60 & $+2.36\;[+1.04,+3.67]$ \\
\bottomrule
\end{tabular*}
\end{wraptable}

Buffer-capacity comparisons evaluate sensitivity to class residency.
In the CIFAR-100 LT10 buffer sweep, class deactivations fall from $16.1$
to zero from capacity $200$ to $5{,}120$, while RPR shortens gaps throughout.
Its mean accuracy difference varies non-monotonically, with
endpoint estimates of $0.30$ and $1.23$ points (Appendix~\ref{app:review-followup}). Storage
controls class availability, and retrieval controls visits during residency.

The pretrained ViT controls evaluate stream and replay-budget dependence.
On balanced streams with buffer $500$ and $b=32$, RPR changes accuracy by
$-1.09$ points on CIFAR-100 and $+0.22$ on ImageNet-R. Smaller replay batches
have longer gap tails, but show no material positive effect on these balanced
streams (Table~\ref{tab:gap-opportunity}). The gap statistics summarise the
evolving resident set, not a stationary full-vocabulary process.

On LT10 streams with the same buffer and $b=8$, RPR increases accuracy by
$2.08$ and $1.53$ points on the respective datasets, with both intervals
above zero. Raising $b$ to $32$ reduces
the independent-draw gap $99$th percentile from $34.2$ to $8.0$ steps on
CIFAR-100 and from $61.4$ to $15.6$ on ImageNet-R. The corresponding accuracy
effects are $-0.14$ and $+0.17$ points, and both intervals include zero.
Within the same ten seeds, the RPR--i.i.d.\ effect decreases by $2.22$
points ($95\%$ CI $[-3.81,-0.63]$) and $1.37$ points ($[-2.01,-0.72]$).
Mean accuracy across the ten task-boundary evaluations shows the same
small- versus large-batch pattern (Appendix~\ref{sec:trajectory}).

Pretrained effects depend on stream and replay budget. Three of four LT10 ViT
contrasts with fixed cycling include zero; the fourth is strictly positive
before rounding (Figure~\ref{fig:mechanism-summary}b). These controls establish
neither small batches nor imbalance as necessary. Readout and MIR controls appear in
Appendices~\ref{sec:emarobust} and~\ref{app:additional}.

\par
\endgroup
\section{Conclusion and Limitations}
\label{sec:conclusion}

Class-balanced retrieval does not determine replay spacing. On a fixed
resident set, RPR bounds rehearsal gaps and renews class co-occurrence without
additional replay examples or forward passes. It improves final accuracy
across the primary ER-ACE storage settings, while pretrained gains concentrate
in the tested LT10 small-batch settings. Linear-head diagnostics are consistent
with accumulated bias pressure during joint absence, but the controls do not
isolate the gap tail. Marginal preservation under churn remains unestablished,
and effects depend on the learner and regime.

These conclusions have a defined scope. The fixed-set gap bound is vacuous
when $b\ge C$ and is conditional on storage under churn. The CIFAR-10 control
at $b=32$, balanced pretrained streams, and matched LT10 controls at $b=32$
show no material positive gain. Evidence covers single-pass CIFAR-100,
Tiny-ImageNet, and ImageNet-R, not multi-epoch, offline, or ImageNet-scale
continual learning. Reused-pass and fixed-cycle controls change gaps and
co-occurrence jointly, and the linear-head bias diagnostic does not fully
explain cosine-head or cross-host results.

\clearpage
\section{AI Use Statement}

We use AI as auxiliary tools for linguistic refinement, restructuring and
formatting text and tables, and assistance with experiment orchestration,
analysis, and verification code. The authors determine the research
questions, methodology, experimental design, statistical criteria, and
interpretation; they review generated code and recompute every reported
quantity from per-run records. AI output does not constitute evidence and
requires author verification. The authors retain all decisions about content,
wording, analysis, and presentation.

\bibliography{iclr2027_conference}
\bibliographystyle{iclr2027_conference}

\clearpage
\appendix
\section*{Appendix}
\section{Reproducibility Statement}
\label{app:reproducibility}

Every experiment records its configuration, command line, random seed,
evaluation split, and campaign identifier. The $256$ validation runs comprise
$96$ hybrid-schedule runs, $96$ replay-scope runs, and $64$ stream-regime
runs; none contributes to a test contrast. EMA and live-readout arms are
separate training runs rather than repeated evaluations of one checkpoint.
For the long-tailed construction,
$r\in\{0,\ldots,C_{\mathrm{task}}-1\}$, $n_{\max}$ is the largest original
class count in that task, and \texttt{round} uses ties-to-even rounding. The
permutation seed is the run's imbalance seed plus $1009$ times the zero-based
task index. The two hybrid candidates assign $16$ or $24$ of $32$ replay
slots to RPR and fill the remainder by an independent balanced draw.
The release includes the scheduler, integrations for each continual-learning
method, experiment configurations, and analysis scripts. Test comparisons
use paired seeds, while validation runs remain separate from test evaluation.
The replay-budget extension reuses the ten LT10 seeds to estimate a paired
$b=32$ minus $b=8$ interaction and is not treated as an independent
confirmation block.
Unit tests cover the scheduler and its limiting cases.

\section{Ethics Statement}

This work trains image classifiers on CIFAR-10, CIFAR-100, Tiny-ImageNet, and
ImageNet-R, all established public benchmarks. We collect no new data,
recruit no human subjects, and use no annotations about personal identity.
As with other image benchmarks, these datasets may reflect biases in their
source data. The methods in this study change which classes a learner
rehearses, and Section~\ref{sec:imbalance} shows that their effects depend on
the stream and replay budget. Rare classes can benefit in the from-scratch
long-tailed settings but can also lose accuracy in a pretrained control.
Retrieval policies of
this kind should therefore be evaluated per class rather than only by
aggregate accuracy when class-level disparities matter.

\section{Measured rehearsal gaps}
\label{app:gaps}

Table~\ref{tab:gaps} isolates the scheduling statistics from learning by
holding the resident class set and balanced memory fixed for $3{,}000$ steps.
The comparison between class-balanced retrieval (\textsc{cb}) and RPR
holds the long-run class-visit marginal fixed; uniform example retrieval
(\textsc{uni}) provides an additional reference. The balanced draw takes a
new uniform random class permutation each step, whereas RPR follows
Algorithm~\ref{alg:cursor}. Uniform retrieval samples example slots and
can select multiple examples of the same class within a batch, so its
class-inclusion probability need not match that of the two class-based
policies. The mean describes how
often a class returns, whereas the upper quantiles and maximum describe
the long absences that the mean alone leaves unconstrained.

\begin{table}[htbp]
\caption{Rehearsal-gap distributions on a fixed balanced memory. All gaps are in steps.}
\label{tab:gaps}
\centering
\small
\begin{tabular}{rr l rrrr r}
\toprule
$C$ & $b$ & policy & mean & p50 & p95 & p99 & $\max_{c,k} G_c(k)$ \\
\midrule
\multirow{3}{*}{100} & \multirow{3}{*}{8}
  & \textsc{uni} & 12.87 & 9 & 38 & 57 & 120 \\
& & \textsc{cb}  & 12.45 & 9 & 36 & 56 & 117 \\
& & \textsc{rpr} & 12.50 & 12 & \textbf{21} & \textbf{23} & \textbf{25} $=2\lceil 100/8\rceil-1$ \\
\midrule
\multirow{3}{*}{100} & \multirow{3}{*}{32}
  & \textsc{uni} & 3.62 & 3 & 10 & 15 & 32 \\
& & \textsc{cb}  & 3.12 & 2 & 8 & 12 & 37 \\
& & \textsc{rpr} & 3.12 & 3 & \textbf{5} & \textbf{6} & \textbf{7} $=2\lceil 100/32\rceil-1$ \\
\midrule
\multirow{3}{*}{200} & \multirow{3}{*}{8}
  & \textsc{uni} & 25.21 & 18 & 75 & 115 & 307 \\
& & \textsc{cb}  & 24.82 & 17 & 73 & 112 & 263 \\
& & \textsc{rpr} & 25.00 & 25 & \textbf{42} & \textbf{46} & \textbf{49} $=2\lceil 200/8\rceil-1$ \\
\midrule
\multirow{3}{*}{200} & \multirow{3}{*}{32}
  & \textsc{uni} & 6.71 & 5 & 19 & 29 & 75 \\
& & \textsc{cb}  & 6.24 & 4 & 18 & 27 & 64 \\
& & \textsc{rpr} & 6.25 & 6 & \textbf{10} & \textbf{11} & \textbf{13} $=2\lceil 200/32\rceil-1$ \\
\bottomrule
\end{tabular}
\end{table}

Across all four configurations, RPR and \textsc{cb} have nearly identical
mean gaps, but RPR approximately halves the $99$th percentile. For example,
at $C=200$ and $b=8$, the mean changes from $24.82$ to $25.00$ steps while
the maximum falls from $263$ to $49$. Each RPR maximum equals the static
bound in Proposition~\ref{prop:gap}. These measurements illustrate the intended
change in the gap distribution; they do not by themselves establish an
accuracy gain. Appendix~\ref{app:dyngaps} tests the corresponding guarantee
when the resident set changes during learning.

\section{The schedule in pseudocode}
\label{app:algorithm}
\begin{algorithm}[H]
\caption{Randomised-pass replay as used in the experiments.
The class deck persists across calls; example lists are freshly shuffled
at each replay step. The resident set is read from the current memory; no
future class identities or total class count are given to the algorithm.}
\label{alg:cursor}
\begin{algorithmic}[1]
\STATE {\bfseries state:} deck $D\leftarrow[\,]$, known classes $K\leftarrow\emptyset$
\STATE {\bfseries input:} memory $M$, requested replay size $b$
\STATE $n\leftarrow\min(b,|M|)$; $B\leftarrow[\,]$
\IF{$n=0$}
  \STATE {\bfseries return} $B$
\ENDIF
\STATE $\mathcal{K}\leftarrow\mathcal{C}(M)$
\STATE $D\leftarrow[\,c\in D : c\in\mathcal{K}\,]$
\FORALL{$c\in\mathcal{K}\setminus K$}
  \STATE insert $c$ into $D$ at a uniformly random position
\ENDFOR
\STATE $K\leftarrow\mathcal{K}$
\FORALL{$c\in\mathcal K$}
  \STATE $H_c\leftarrow$ uniformly shuffled list of stored example indices of class $c$
\ENDFOR
\STATE $g\leftarrow0$
\WHILE{$|B|<n$ and $g\le|M|$}
  \STATE $V\leftarrow\operatorname{ClassSlots}(\mathcal K,n-|B|,D)$; $g\leftarrow g+1$
  \FORALL{$c\in V$}
    \IF{$H_c\ne[\,]$ and $|B|<n$}
      \STATE pop one index from $H_c$ and append it to $B$
    \ENDIF
  \ENDFOR
\ENDWHILE
\STATE {\bfseries return} the examples indexed by $B$
\end{algorithmic}
\end{algorithm}

\noindent\textbf{ClassSlots$(\mathcal K,k,D)$.}
Let $C=|\mathcal K|$ and $q=\lfloor k/C\rfloor$. First output $q$ complete,
independently shuffled class permutations, without advancing $D$.
For the remaining $k-qC$ slots, remove class identifiers from $D$. Whenever $D$ empties,
refill it with a uniformly shuffled permutation of $\mathcal K$, moving
classes already output in this partial round to the end while preserving
their relative order. Persist the remaining deck for the next call.
If $k<C$, only this persistent-deck part operates. If $k\ge C$, the complete
rounds ensure that every resident class is visited before the scheduled
remainder. Removing example indices from $H_c$ samples without replacement
within a replay batch; the same stored example can return at a later step.
The bounded loop matches the implementation's guard against non-progress and
may return fewer than $n$ examples when a large-batch request repeatedly
visits classes whose within-step lists are exhausted. This edge case is
outside the fixed-set analysis below, which assumes $b\le C$ and enough
stored examples for every requested class slot.

\section{Additional Analysis and Proofs}
\label{app:theory}

We denote RPR by $\pi_{\mathrm{rp}}$ in the following derivations.

\subsection{Static Schedule Diagnostics}

Table~\ref{tab:gaps} reports the static schedule diagnostics. Across the four
configurations, the absolute deviation of the mean gap from $C/b$ is at most
$0.005$ for RPR and $0.18$ for the i.i.d.\ balanced draw. The corresponding
i.i.d.-to-RPR ratios for the $99$th-percentile gap are $2.0$--$2.5$, and
the ratios for the observed maximum are $4.7$--$5.4$. In all four
configurations, the observed RPR maximum equals $2\lceil C/b\rceil-1$.

\subsection{Composition Dependence and the ER-ACE Mask}
\label{app:composition}

The per-step marginal corresponding to Equation~\eqref{eq:timeavg} is
\begin{equation}
  \mu_c^{(t)}(\pi)=
  \Pr[c\in\mathcal{C}(B_t)\mid\mathcal{F}_t],
  \qquad c\in\mathcal{K}_t.
  \label{eq:marginal}
\end{equation}
Equality of time-averaged marginals does not imply equality of these
history-conditional laws.

Definition~\ref{def:blind} and Equation~\eqref{eq:blindexp} characterise the
one-step replay gradient. ER-ACE introduces an additional asymmetry through
the incoming loss.
The from-scratch ResNet-18 uses training-mode BatchNorm, so its forward map
is batch-dependent and Equation~\eqref{eq:blindexp} is an idealised scope
statement for that configuration. The ViT instead uses per-example LayerNorm.

ER-ACE's incoming loss retains current-batch and unseen output classes:
\begin{equation}
  \ell^{\mathrm{on}}(\theta;A_t)
  =\sum_{(x,y)\in A_t}-\log
  \frac{\exp z_y(x)}
  {\sum_{c\in\mathcal{C}(A_t)\cup\mathcal{U}_t}\exp z_c(x)},
  \qquad \mathcal{U}_t=\mathcal{Y}\setminus\mathcal{S}_t.
  \label{eq:ace}
\end{equation}
Here $\mathcal{Y}$ denotes the host classifier's output vocabulary; RPR
does not inspect that vocabulary or use unseen labels to construct its
deck. For $c\in\mathcal{S}_t\setminus\mathcal{C}(A_t)$,
\begin{equation}
  \frac{\partial\ell^{\mathrm{on}}}{\partial z_c}=0.
  \label{eq:nogradient}
\end{equation}
For a resident replay gap of $g>1$ starting at $t=\tau_c(k)$, if $c$ is also
absent from every interior incoming batch, Equation~\eqref{eq:onesided} gives
\begin{equation}
  \sum_{s=1}^{g-1}\frac{\partial\mathcal{L}_{t+s}}{\partial\beta_c}
  =\sum_{s=1}^{g-1}\sum_{x\in B_{t+s}}p_c(x)>0.
  \label{eq:accumulate}
\end{equation}
There are $g-1$ interior steps, since the final endpoint is a replay visit.
The sum characterises accumulated bias-gradient pressure. Its magnitude is
not determined by gap length alone, and Equation~\eqref{eq:accumulate} does
not imply a universal relationship between gap length and final accuracy.

\subsection{From Logit Gradients to Parameter Updates}
\label{app:jacobian}

The logit-space gradient underlying Equation~\eqref{eq:onesided} is
$\partial\ell^{\mathrm{re}}/\partial z_c(x)=p_c(x)>0$ at each replay input
during joint absence. The sign of
$\partial\ell^{\mathrm{re}}/\partial z_c(x)$ does not determine the sign of
the parameter-induced change in $z_c(x')$.
Under a gradient step on $\theta$ the induced logit movement at
an arbitrary input $x'$ is, to first order,
\begin{equation}
  \Delta z_c(x')
  \;\approx\;
  -\,\eta\,J_c(x')^{\!\top}\nabla_\theta\mathcal{L},
  \qquad
  J_c(x')=\nabla_\theta z_c(x'),
  \label{eq:jacobian}
\end{equation}
so a positive $\partial\ell/\partial z_c$ at the replayed points does not by
itself fix the sign of $\Delta z_c(x')$ at other points: the shared
representation couples classes and inputs through
$J_c(x')^{\!\top}J_{c'}(x)$, the empirical neural tangent kernel
\citep{jacot2018neural}. For the class-specific parameters of the final
linear layer, the bias gradient has a fixed sign during joint absence.
Writing $z_c(x)=w_c^{\!\top}h_\theta(x)+\beta_c$, the bias
gradient is $\partial\mathcal{L}/\partial\beta_c=\sum_{x\in B_t}p_c(x)>0$ at
every step of joint absence. Under plain gradient descent, the bias update
is negative at each such step. Momentum, adaptive preconditioning, and
weight decay add terms that depend on past gradients, the preconditioner,
and the current parameter value; the gradient sign alone does not establish
an optimizer-independent update sign. At a single plain
gradient step, the update of $w_c$ likewise has non-positive projection
along its current gradient, but that direction can change across steps.
The shared-trunk contribution remains sign-indefinite.
Section~\ref{sec:gapdrift} measures the realised bias displacement.

\subsection{Proof of the Churn-Conditional Bound}

We prove Proposition~\ref{prop:churn}, the counterpart of
Equation~\eqref{eq:bound} for a resident class set that changes over time.

\begin{proof}[Proof of Proposition~\ref{prop:churn}]
Count from the end of step $\tau_c(k)$, rather than from the position of
$c$ inside its minibatch. If the current pass still contains $c$, at most
$C_{\max}$ existing slots remain up to its visit. Otherwise, at most
$C_{\max}-1$ existing slots remain in that pass, followed by at most
$C_{\max}$ slots up to $c$ in the next pass. An insertion adds at most
one further slot before this visit; deletions remove slots. Hence at most
$2C_{\max}-1+I$ slots remain. A refill within a later step cannot defer
$c$ unless that step already visits it, in which case the waiting interval
ends. Each intervening step either supplies $b$ distinct class slots or,
when fewer than $b$ classes remain resident, visits every resident class
and hence $c$. Since the count starts at a step boundary, at most
$\lceil(2C_{\max}-1+I)/b\rceil$ further steps are required.
\end{proof}

The bound requires $c$ to remain resident. If $c\notin\mathcal{K}_t$ during an
interval, no retrieval policy can include it in $B_t$, so no retrieval rule
can bound the corresponding wall gap.

\subsection{Additional Scheduling Results}

The following results give a worst-case lower bound, characterise schedules
that attain the lower bound exactly when $b\mid C$, and describe the
stream-length dependence of the maximum rehearsal gap. A final result gives
a convex cost bound at fixed mean rehearsal gap.

\begin{proposition}[Worst-case gap lower bound]
\label{prop:floor}
For any recurrent retrieval policy on a fixed resident set,
\begin{equation}
  \max_{c}\ \sup_k G_c(k)\ \ge\ \Big\lceil \tfrac{C}{b}\Big\rceil .
  \label{eq:floor}
\end{equation}
\end{proposition}

\begin{proof}
A window of $w$ consecutive steps contains at most $bw$ replay class slots
and therefore at most $bw$ distinct classes. Take $w=\lceil C/b\rceil-1$;
then $bw<C$ whether or not $b$ divides $C$, so some class $c$ is absent from
the whole window. The adjacent replay visits to $c$ lie on opposite sides of
the window, giving $G_c\ge w+1$.
\end{proof}

Combining Proposition~\ref{prop:floor} with Proposition~\ref{prop:gap} gives
an upper bound of $2-1/\lceil C/b\rceil$ on RPR's worst-case ratio to the
lower bound. The geometric gap
distribution in Equation~\eqref{eq:geometric} has unbounded support. The next
proposition characterises schedules that attain the lower bound exactly when
$b\mid C$.

\begin{proposition}[Exact lower-bound attainment implies a fixed partition]
\label{prop:tradeoff}
On a fixed resident set of $C$ classes, suppose a bi-infinite schedule visits
exactly $b$ distinct classes per step, $b\mid C$, and attains
$G_c(k)=C/b$ for every successive pair of visits. Then there is a fixed
partition of the resident class set into $C/b$ blocks of size $b$ such that the replay batch's class set is always one block. Two
classes in the same block co-occur at every visit; two classes in different
blocks never co-occur.
\end{proposition}

\begin{proof}
Write $m=C/b$. A uniform gap of $m$ means class $c$ is visited exactly at the
steps congruent to some fixed $\phi(c)$ modulo $m$. Hence the class set at
step $t$ is $\{c:\phi(c)\equiv t\}$, which depends on $t$ only through
$t \bmod m$ and partitions the resident class set into the $m$ level sets
of $\phi$,
each necessarily of size $b$ since every step visits $b$ classes.
\end{proof}

For a batch-coupled loss, the fixed partition restricts class $c$ to
co-occurrence with the $b-1$ other classes in its block.
Proposition~\ref{prop:tradeoff} applies only when $G_c(k)=C/b$ for every
class and every visit; it does not characterise schedules whose gaps are only
close to the lower bound. RPR samples a new class permutation each pass,
whereas deterministic class cycling induces the fixed partition when
$b\mid C$.

For Equation~\eqref{eq:paircoverage}, condition on the position of $c$ in a
uniformly permuted pass. Exactly $b-1$ of the other $C-1$ positions share its
replay batch, so $d$ shares the replay batch with $c$ with probability
$(b-1)/(C-1)$. That batch occupies one
of the $C/b$ steps in the pass, giving the additional factor $b/C$ in the
time-averaged pair frequency.

\begin{proposition}[Stream-length scaling of the maximum gap]
\label{prop:separation}
On a fixed resident set with $b<C$, over the first $n$ rehearsal gaps of each
class,
the i.i.d.\ balanced draw satisfies
\begin{equation}
  \mathbb{E}\Big[\max_{c,k\le n} G_c(k)\Big]
  \;=\;O\!\left(
  \frac{\log (nC)}{\log\!\big(1/(1-p)\big)}\right),
  \qquad p=\tfrac{b}{C},
  \label{eq:evt}
\end{equation}
and the maximum over the $n$ independent gaps of any fixed class is
$\Omega(\log n/\log(1/(1-p)))$. The i.i.d.\ maximum therefore has
logarithmic dependence on $n$, whereas $\pi_{\mathrm{rp}}$ satisfies
$\max_{c,k}G_c(k)\le 2\lceil C/b\rceil-1$ independently of $n$.
When $b\mid C$, as $m=C/b\to\infty$,
\begin{equation}
  \frac{\operatorname{Var}\big[G_c^{\mathrm{iid}}\big]}
       {\operatorname{Var}\big[G_c^{\mathrm{rp}}\big]}
  \;=\;\frac{(1-p)/p^{2}}{(m^{2}-1)/6}\;\longrightarrow\;6 ,
  \label{eq:varratioapp}
\end{equation}
so the limiting ratio of standard deviations is $\sqrt{6}\approx 2.449$.
\end{proposition}

The upper bound in Equation~\eqref{eq:evt} follows by a union bound over
$nC$ geometric tails and does not assume independence across classes; the
lower bound uses the renewal gaps of one class.
Equation~\eqref{eq:varratioapp} follows from $G_c^{\mathrm{rp}}
= m-i+j$ with $i,j$ the batch positions of $c$ in two consecutive passes, which are
independent and uniform on $\{1,\dots,m\}$ when $b \mid C$.

For the four configurations, the geometric extreme-value scales are
$121$, $30$, $247$, and $66$, respectively; the measured maxima are
$117$, $37$, $263$, and $64$. The measured standard-deviation ratios are
$2.41$, $2.42$, $2.38$, and $2.43$.
The finite-$m$ variance expression is exact when $b$ divides $C$; elsewhere a
partial pass changes the finite-sample distribution, so $\sqrt6$ is an
asymptotic reference rather than an exact prediction.

Proposition~\ref{prop:separation} predicts logarithmic growth of the
i.i.d.\ maximum gap with the number of draws, while the RPR bound is
independent of stream length. The proposition concerns the scheduling
statistic and does not imply a monotone relationship with accuracy.

The following proposition considers a convex cost functional over rehearsal
gaps at fixed mean gap. A balanced long-run class marginal gives
$\mathbb{E}[G_c]=C/b$ once finite boundary effects vanish.

\begin{proposition}[Jensen bound for convex gap cost]
\label{prop:jensen}
Let $h:\mathbb{R}_{\ge0}\to\mathbb{R}_{\ge0}$ be non-decreasing and convex, modelling
the cost of a rehearsal gap of $g$ steps. Consider recurrent
retrieval policies for which the event-average gap exists and satisfies
$\mathbb{E}[G_c]=C/b$. Their expected cost obeys
\begin{equation}
  \mathbb{E}\big[h(G_c)\big]\ \ge\ h\big(C/b\big),
  \label{eq:jensen}
\end{equation}
If $h$ is strictly convex on the convex hull of the support of $G_c$,
equality holds if and only if $G_c$ is almost surely constant. Under mere
convexity, nonconstant gaps can also attain equality when $h$ is affine on
their support.
\end{proposition}

Equation~\eqref{eq:jensen} is Jensen's inequality and is conditional on the
choice of $h$. It does not assume that accuracy is convex in the rehearsal
gap. Policies with the same marginal may also differ in temporal dependence
and class co-occurrence; the proposition isolates gap dispersion within the
specified cost model.

\section{Robustness to the evaluation-time weight average}
\label{sec:emarobust}

We compare separately trained arms with and without the averaged readout of
Section~\ref{sec:protocol} to assess its interaction with retrieval. At buffer
$500$, the $320$-run block crosses four dataset--stream cells with live and
EMA readouts under reservoir storage and replay batch $32$. Removing the exponential moving
average (EMA) increases the relative effect of the balanced draw from
$+6.5$--$+11.5\%$ to $+12.8$--$+17.9\%$. These comparisons show that the
relative balanced-draw effect depends on the readout in these settings.

At the primary operating point, a separate $120$-run block evaluates
reservoir storage, buffer $5{,}120$, and $b=8$ without the EMA, paired to the
seeds and cells of
Table~\ref{tab:seams}. Removing the readout reduces accuracy by $1.6$ to
$5.2$ points in eleven of the twelve arms. The balanced-draw effects over
uniform retrieval are $-0.22$, $+1.08$, $+0.36$, and $+0.65$ points for
CIFAR-100 balanced, CIFAR-100 LT10, CIFAR-100 LT100, and Tiny-ImageNet LT10,
respectively; only the LT10 CIFAR-100 effect is material. In contrast,
RPR remains material in all four cells, with gains of $+2.41$, $+2.22$,
$+1.39$, and $+1.92$ points ($+16.7\%$ to $+35.7\%$ relative) over uniform
retrieval. Its margins over the balanced draw are $+2.63$, $+1.14$,
$+1.03$, and $+1.28$ points, with at least eight of ten paired differences
positive in each cell. These runs have positive mean RPR effects without
the averaged readout; the balanced-draw effects are material in only one cell.

\section{Accuracies for crossed storage and retrieval policies}
\label{app:accuracies}

Table~\ref{tab:seams} provides the absolute accuracies underlying the
storage--retrieval contrasts in Table~\ref{tab:contrasts}. Every arm uses
ER-ACE with a ResNet-18 initialised from scratch, buffer $5{,}120$, replay
batch $b=8$, and one stream pass. Entries report mean test accuracy and
sample standard deviation over ten paired seeds; bold identifies the
highest mean within each stream. LT10 and LT100 denote exponential
long-tailed streams. Reservoir sampling ($\reservoir$) and Balanced
Reservoir Sampling (BRS), which evicts from the largest class, form the
two storage policies. Comparing columns
within a row holds storage fixed and changes retrieval; comparing the two
rows within a stream holds retrieval fixed and changes storage. This
factorial comparison separates which examples remain available from when
their classes return to the learner. The reported standard deviations
describe individual-arm variability; paired inference appears in
Appendix~\ref{app:statistics}.

\begin{table}[htbp]
\caption{Final test accuracy (\%) for crossed storage and retrieval policies.}
\label{tab:seams}
\centering
\scriptsize
\begin{tabular}{l l cccc}
\toprule
& & \multicolumn{4}{c}{retrieval $\pi_{\mathrm{draw}}$} \\
\cmidrule(l){3-6}
stream & storage $\pi_{\mathrm{store}}$ & \textsc{uni} & \textsc{iid} & \textsc{fix} & \textsc{rpr} \\
\midrule
\multirow{2}{*}{CIFAR-100, balanced}
 & $\reservoir$ & $18.75{\pm}0.40$ & $19.42{\pm}0.58$ & $18.55{\pm}0.44$ & $\mathbf{20.61{\pm}0.67}$ \\
 & \textsc{brs} & $18.90{\pm}0.63$ & $19.33{\pm}0.58$ & $18.64{\pm}0.59$ & $20.55{\pm}0.79$ \\
\midrule
\multirow{2}{*}{CIFAR-100, lt$10$}
 & $\reservoir$ & $\phantom{0}9.44{\pm}0.46$ & $11.23{\pm}0.46$ & $12.05{\pm}0.42$ & $12.63{\pm}0.41$ \\
 & \textsc{brs} & $11.30{\pm}0.65$ & $11.43{\pm}0.52$ & $11.78{\pm}0.59$ & $\mathbf{12.81{\pm}0.52}$ \\
\midrule
\multirow{2}{*}{CIFAR-100, lt$100$}
 & $\reservoir$ & $\phantom{0}6.69{\pm}0.73$ & $\phantom{0}8.00{\pm}0.64$ & $\phantom{0}8.26{\pm}0.72$ & $\phantom{0}8.72{\pm}0.34$ \\
 & \textsc{brs} & $\phantom{0}7.57{\pm}1.11$ & $\phantom{0}8.36{\pm}0.36$ & $\phantom{0}8.80{\pm}0.55$ & $\mathbf{\phantom{0}8.95{\pm}0.47}$ \\
\midrule
\multirow{2}{*}{Tiny-ImageNet, lt$10$}
 & $\reservoir$ & $\phantom{0}7.09{\pm}0.55$ & $\phantom{0}7.76{\pm}0.30$ & $\phantom{0}8.45{\pm}0.54$ & $\phantom{0}9.43{\pm}0.59$ \\
 & \textsc{brs} & $\phantom{0}8.08{\pm}0.60$ & $\phantom{0}7.92{\pm}0.45$ & $\phantom{0}8.56{\pm}0.45$ & $\mathbf{\phantom{0}9.69{\pm}0.50}$ \\
\bottomrule
\end{tabular}
\end{table}

RPR has the highest mean accuracy within every storage--stream row, and
BRS with RPR gives the best mean in all three long-tailed streams. On
balanced CIFAR-100, reservoir storage with RPR is slightly higher than
BRS with RPR ($20.61$ versus $20.55$), so storage balancing is not uniformly
beneficial. RPR also has higher mean accuracy than independent balanced
retrieval within every BRS row.

\section{Paired statistical tests}
\label{app:statistics}

Table~\ref{tab:statistics} quantifies the uncertainty in the RPR contrasts
of Table~\ref{tab:contrasts}. Pairing compares retrieval policies under the
same seed, so $s_\Delta$ describes variation in the within-seed difference,
not the accuracy variation of either arm separately. The test statistic
has nine degrees of freedom, $p$ is its two-sided Student-$t$ probability,
and $k$ counts positive paired differences out of ten. Each Holm family
contains the four streams for one storage policy and one comparator;
the adjustment does not pool all sixteen tests into a single family.

\begin{table}[htbp]
\caption{Paired tests of RPR against independent balanced retrieval and fixed cycling.}
\label{tab:statistics}
\centering
\scriptsize
\resizebox{\textwidth}{!}{%
\begin{tabular}{lllrrrrr}
\toprule
storage & comparator & stream & $\Delta\pm s_\Delta$ & $t(9)$ & $p$ & Holm $p$ & $k$ \\
\midrule
\multirow{8}{*}{$\reservoir$}
 & \multirow{4}{*}{i.i.d.} & CIFAR-100 bal.   & $1.191\pm0.616$ & 6.114 & $1.76{\times}10^{-4}$ & $3.52{\times}10^{-4}$ & 10 \\
 & & CIFAR-100 lt$10$  & $1.400\pm0.476$ & 9.294 & $6.56{\times}10^{-6}$ & $1.97{\times}10^{-5}$ & 10 \\
 & & CIFAR-100 lt$100$ & $0.723\pm0.605$ & 3.780 & $4.35{\times}10^{-3}$ & $4.35{\times}10^{-3}$ & 9 \\
 & & Tiny-IN lt$10$    & $1.671\pm0.541$ & 9.763 & $4.37{\times}10^{-6}$ & $1.75{\times}10^{-5}$ & 10 \\
\cmidrule(l){2-8}
 & \multirow{4}{*}{fixed cycle} & CIFAR-100 bal.   & $2.058\pm0.736$ & 8.847 & $9.83{\times}10^{-6}$ & $3.93{\times}10^{-5}$ & 10 \\
 & & CIFAR-100 lt$10$  & $0.588\pm0.645$ & 2.883 & $1.81{\times}10^{-2}$ & $3.62{\times}10^{-2}$ & 8 \\
 & & CIFAR-100 lt$100$ & $0.466\pm0.709$ & 2.080 & $6.73{\times}10^{-2}$ & $6.73{\times}10^{-2}$ & 8 \\
 & & Tiny-IN lt$10$    & $0.973\pm0.720$ & 4.273 & $2.07{\times}10^{-3}$ & $6.21{\times}10^{-3}$ & 9 \\
\midrule
\multirow{8}{*}{\textsc{brs}}
 & \multirow{4}{*}{i.i.d.} & CIFAR-100 bal.   & $1.221\pm1.110$ & 3.477 & $6.97{\times}10^{-3}$ & $6.97{\times}10^{-3}$ & 10 \\
 & & CIFAR-100 lt$10$  & $1.383\pm0.505$ & 8.657 & $1.17{\times}10^{-5}$ & $3.52{\times}10^{-5}$ & 10 \\
 & & CIFAR-100 lt$100$ & $0.583\pm0.463$ & 3.979 & $3.21{\times}10^{-3}$ & $6.42{\times}10^{-3}$ & 9 \\
 & & Tiny-IN lt$10$    & $1.773\pm0.560$ & 10.005 & $3.56{\times}10^{-6}$ & $1.43{\times}10^{-5}$ & 10 \\
\cmidrule(l){2-8}
 & \multirow{4}{*}{fixed cycle} & CIFAR-100 bal.   & $1.911\pm0.608$ & 9.937 & $3.77{\times}10^{-6}$ & $1.51{\times}10^{-5}$ & 10 \\
 & & CIFAR-100 lt$10$  & $1.034\pm0.664$ & 4.924 & $8.20{\times}10^{-4}$ & $1.64{\times}10^{-3}$ & 9 \\
 & & CIFAR-100 lt$100$ & $0.143\pm0.394$ & 1.147 & $2.81{\times}10^{-1}$ & $2.81{\times}10^{-1}$ & 6 \\
 & & Tiny-IN lt$10$    & $1.130\pm0.404$ & 8.844 & $9.85{\times}10^{-6}$ & $2.95{\times}10^{-5}$ & 10 \\
\bottomrule
\end{tabular}%
}
\end{table}

Every RPR--i.i.d.\ comparison remains significant at the $5\%$ level after
this adjustment, under both reservoir and BRS storage. RPR also exceeds
fixed cycling in six of the eight adjusted tests. Both exceptions concern
CIFAR-100 LT100, where the positive mean differences do not establish an
additional benefit over fixed cycling. Across these comparisons, evidence for
RPR over independent draws is more consistent than evidence for RPR over fixed cycling.

\section{Measured gaps inside the online runs}
\label{app:dyngaps}

Table~\ref{tab:dyngaps} extends the static diagnostic to an evolving replay
buffer. Its class count $C$ denotes the benchmark vocabulary, not an input
available to RPR in advance; the scheduler uses only the current resident
set. All runs use ER-ACE, buffer $5{,}120$, replay batch $b=8$, and ten
seeds per cell; entries average the per-run statistics across seeds.
Because this set grows and can lose classes, the aggregate mean gap
need not equal the stationary value $C/b$. The resident maximum tests the
churn-conditional guarantee, while the wall maximum also counts time when
storage makes a class unavailable to every retrieval rule. The violation
count records steps whose resident gap exceeds Equation~\eqref{eq:churnbound}
using the run's own $C_{\max}$ and insertion count. Class deactivations
record departures from memory.

\begin{table}[htbp]
\caption{Online rehearsal-gap statistics (steps), bound exceedances, and class deactivations.}
\label{tab:dyngaps}
\centering
\small
\setlength{\tabcolsep}{5.5pt}
\begin{tabular}{l r l rrr r r rr}
\toprule
stream & $C$ & policy & mean & p95 & p99 & res.\ max & wall max & viol. & deact. \\
\midrule
\multirow{3}{*}{CIFAR-100 LT10} & \multirow{3}{*}{100}
  & \textsc{cb}  & 6.63 & 22.0 & 37.1 & 81.3 & 81.3 & $427.4$ & 0.0 \\
& & \textsc{fix} & 6.79 & 12.0 & 13.0 & 13.0 & 13.0 & $\mathbf{0}$ & 0.0 \\
& & \textsc{rpr} & 6.81 & \textbf{15.6} & \textbf{19.6} & \textbf{23.9} & \textbf{23.9} & $\mathbf{0}$ & 0.0 \\
\midrule
\multirow{3}{*}{Tiny-ImageNet LT10} & \multirow{3}{*}{200}
  & \textsc{cb}  & 13.18 & 44.7 & 75.7 & 175.9 & 175.9 & $1032.8$ & 0.8 \\
& & \textsc{fix} & 13.57 & 25.0 & 25.0 & 25.0 & 25.0 & $\mathbf{0}$ & 0.8 \\
& & \textsc{rpr} & 13.58 & \textbf{31.5} & \textbf{39.3} & \textbf{48.1} & \textbf{48.1} & $\mathbf{0}$ & 0.8 \\
\midrule
\multirow{3}{*}{Tiny-ImageNet LT100} & \multirow{3}{*}{200}
  & \textsc{cb}  & 12.58 & 42.5 & 69.7 & 152.6 & 152.6 & $521.9$ & 7.4 \\
& & \textsc{fix} & 13.14 & 23.7 & 24.1 & 24.6 & 29.8 & $\mathbf{0}$ & 7.4 \\
& & \textsc{rpr} & 13.19 & \textbf{30.3} & \textbf{37.9} & \textbf{46.3} & 47.8 & $\mathbf{0}$ & 7.4 \\
\bottomrule
\end{tabular}
\end{table}

RPR yields no observed violation of the churn-conditional bound in any of
the three settings, while the independent draw exceeds the same threshold
hundreds of times per run. Fixed cycling has shorter tails than RPR in these
settings. On Tiny-ImageNet LT100, the RPR wall maximum of $47.8$ exceeds its
resident maximum of $46.3$ steps, alongside $7.4$ class deactivations per
run. This distinction makes the guarantee explicitly conditional on storage
retaining the class.

\section{Head, mid and tail accuracy and per-class forgetting}
\label{app:groups}

Table~\ref{tab:groups} examines whether the aggregate scheduling gain
reflects a redistribution of accuracy towards frequent classes. Within each
run, classes are sorted by observed online frequency, with class identifier
breaking ties. The first and second $\lfloor C/3\rfloor$ classes form the
tail and mid groups, and the remainder forms the head group; the resulting
sizes are $33/33/34$ for CIFAR-100 and $66/66/68$ for ImageNet-R and
Tiny-ImageNet. Macro accuracy weights all classes equally. Each entry compares RPR with
i.i.d.\ class-balanced retrieval on the test split, with $95\%$ Student-$t$
intervals over ten paired seeds. The from-scratch ResNet-18 settings use
buffer $5{,}120$ and $b=8$; pretrained ViT-B/16 uses buffer $500$ and $b=32$.
For class $c$, forgetting is its highest accuracy over task-boundary
evaluations minus its final accuracy; the table averages this quantity over
classes. Positive accuracy differences favour RPR; negative forgetting
differences indicate improved retention.

\begin{table}[htbp]
\caption{RPR minus independent balanced retrieval in frequency-group accuracy and forgetting.
Differences are in percentage points; macro accuracy weights classes equally
and therefore differs from task-mean accuracy when class sizes are unequal.}
\label{tab:groups}
\centering
\small
\begin{tabular}{l rrrr r}
\toprule
cell & head & mid & tail & macro & forgetting \\
\midrule
\multicolumn{6}{l}{\emph{from-scratch ResNet-18, buffer $5{,}120$, $b=8$}} \\
CIFAR-100 LT10
  & $+1.20$ & $+1.55$ & $+1.45$ & $\mathbf{+1.40}$ & $\mathbf{-2.70}$ \\
  & \tiny[0.76, 1.65] & \tiny[0.52, 2.58] & \tiny[0.34, 2.56] & \tiny[1.06, 1.74] & \tiny[-3.12, -2.29] \\
Tiny-ImageNet LT10
  & $+1.64$ & $+1.64$ & $+1.73$ & $\mathbf{+1.67}$ & $\mathbf{-2.75}$ \\
  & \tiny[0.69, 2.60] & \tiny[0.81, 2.47] & \tiny[0.61, 2.85] & \tiny[1.28, 2.06] & \tiny[-3.38, -2.11] \\
Tiny-ImageNet LT100
  & $+0.69$ & $+0.82$ & $+1.00$ & $\mathbf{+0.84}$ & $\mathbf{-1.74}$ \\
  & \tiny[-0.02, 1.41] & \tiny[0.33, 1.30] & \tiny[0.50, 1.51] & \tiny[0.55, 1.13] & \tiny[-2.42, -1.05] \\
\midrule
\multicolumn{6}{l}{\emph{pretrained ViT-B/16, buffer $500$, $b=32$}} \\
CIFAR-100 (224px)
  & $-0.02$ & $-1.03$ & $\mathbf{-2.25}$ & $\mathbf{-1.09}$ & $\mathbf{+0.91}$ \\
  & \tiny[-1.51, 1.47] & \tiny[-2.99, 0.92] & \tiny[-3.79, -0.71] & \tiny[-2.03, -0.15] & \tiny[0.08, 1.74] \\
ImageNet-R
  & $+0.14$ & $+0.26$ & $-0.07$ & $+0.11$ & $+0.13$ \\
  & \tiny[-0.98, 1.25] & \tiny[-1.18, 1.70] & \tiny[-1.55, 1.41] & \tiny[-0.84, 1.06] & \tiny[-0.83, 1.08] \\
\bottomrule
\end{tabular}
\end{table}

In the from-scratch settings, all three frequency groups have positive
mean differences, and tail gains are at least as large as head gains.
The intervals exclude zero for every group except the head of
Tiny-ImageNet LT100; macro forgetting decreases in all three settings.
The pretrained $b=32$ controls show a different pattern: CIFAR-100 loses
$2.25$ points on tail classes and has greater forgetting, while every
ImageNet-R interval includes zero. These results support a setting-dependent
benefit rather than a universal improvement in rare-class retention,
consistent with the gap-opportunity analysis in
Table~\ref{tab:gap-opportunity}.

\section{Gap drift, co-occurrence, and buffer diagnostics}
\label{app:review-followup}
\subsection{Bias displacement during joint absence}

Table~\ref{tab:gap-drift} measures the realised classifier-bias change over
episodes in which a resident class is absent from both incoming and replay
batches under masked ER-ACE with a linear classifier. All three diagnostics
in this appendix use the same linear classifier with bias, whereas the primary
storage--retrieval comparison uses the bias-free cosine head. This
distinction is necessary for interpreting the measured bias displacement.
An episode starts at a replay visit and closes at the next replay visit.
An incoming-only return neither closes nor censors it, but that update is not
included in the accumulated displacement because the class is not jointly
absent. Eviction discards the open episode. Columns group completed episodes
by replay-to-replay gap; the displacement sums only jointly absent updates.
Within each seed we average completed episodes in a bin, then average those
per-seed means equally across the ten seeds.
This diagnostic measures bias displacement along the training trajectory
to evaluate the bias-pressure mechanism.

\begin{table}[htbp]
\centering
\caption{Classifier-bias displacement by replay-to-replay gap (steps),
accumulated only over jointly absent updates.}
\label{tab:gap-drift}
\small
\begin{tabular}{llrrrrr}
\toprule
Dataset & Draw & $1$--$4$ & $5$--$8$ & $9$--$16$ & $17$--$32$ & $\geq33$ \\
\midrule
\multirow{2}{*}{CIFAR-100}
 & i.i.d. & $-.00191$ & $-.00340$ & $-.00480$ & $-.00663$ & $-.00926$ \\
 & RPR    & $-.00211$ & $-.00311$ & $-.00418$ & $-.00565$ & --- \\
\midrule
\multirow{2}{*}{Tiny-ImageNet}
 & i.i.d. & $-.00132$ & $-.00230$ & $-.00340$ & $-.00486$ & $-.00728$ \\
 & RPR    & $-.00160$ & $-.00236$ & $-.00311$ & $-.00420$ & $-.00573$ \\
\bottomrule
\end{tabular}
\end{table}

All populated bins have negative-displacement fraction $1.00$, and the
magnitude of the mean displacement increases with episode length under
both retrieval policies. RPR has no completed CIFAR-100 episode in the
$\geq33$-step bin; the dash denotes absence of observations, not zero drift.
These measurements support accumulated bias pressure during long gaps.
However, the bins condition on different episodes and classes, so a
within-bin difference between policies is not an isolated causal estimate
of scheduling, nor does the bias displacement determine the movement of
every logit through the shared representation.

\subsection{Renewal of class co-occurrence}

Table~\ref{tab:renewal-diagnostic} varies how often the class order renews
at buffer size $5{,}120$. The parameter $K$ counts traversals that reuse
one shuffled pass: RPR sets $K=1$, the intermediate schedules set $K=4$
or $16$, and fixed cycling never renews the order. Pair coverage and pair
entropy describe the diversity of within-batch class co-occurrence, while
the gap columns record the remaining temporal variation. Accuracy
differences use independent balanced retrieval as the comparator and
report $95\%$ paired intervals. Both pair statistics cover the full training
stream. Pair coverage is the number of distinct observed replay-class pairs
divided by $\binom{C_{\max}}{2}$, where $C_{\max}$ is the largest resident
class count in the run. Normalised pair entropy is
$-\sum_p q_p\log q_p/\log\binom{C_{\max}}{2}$, where $q_p$ is a pair's share
of all within-batch pair observations and the logarithm is natural.

\begin{table}[htbp]
\centering
\caption{Co-occurrence diversity, rehearsal gaps, and accuracy under different
renewal frequencies. Accuracy differences are in percentage points.}
\label{tab:renewal-diagnostic}
\small
\setlength{\tabcolsep}{4pt}
\begin{tabular}{llrrrr}
\toprule
Dataset & Schedule & Pair coverage & Pair entropy & Gap p99 / max & $\Delta$ accuracy \\
\midrule
\multirow{4}{*}{CIFAR-100}
 & RPR ($K=1$)    & $.739$ & $.892$ & $19.5/24.1$ & $+1.23\ [+.53,+1.92]$ \\
 & Epochal $K=4$  & $.601$ & $.875$ & $17.0/23.2$ & $+1.15\ [+.61,+1.70]$ \\
 & Epochal $K=16$ & $.515$ & $.855$ & $16.1/23.4$ & $+1.04\ [+.29,+1.78]$ \\
 & Fixed          & $.168$ & $.725$ & $13.0/13.0$ & $+0.47\ [-.27,+1.21]$ \\
\midrule
\multirow{4}{*}{Tiny-ImageNet}
 & RPR ($K=1$)    & $.568$ & $.893$ & $39.6/48.5$ & $+1.16\ [+.73,+1.59]$ \\
 & Epochal $K=4$  & $.425$ & $.867$ & $36.6/48.1$ & $+1.20\ [+.78,+1.62]$ \\
 & Epochal $K=16$ & $.364$ & $.841$ & $35.8/47.5$ & $+1.23\ [+.81,+1.65]$ \\
 & Fixed          & $.074$ & $.687$ & $25.0/25.0$ & $+0.84\ [+.52,+1.15]$ \\
\bottomrule
\end{tabular}
\end{table}

Less frequent renewal reduces pair coverage and entropy on both datasets,
but accuracy does not follow a universal monotone ordering. On CIFAR-100,
the mean gain decreases from $1.23$ points for RPR to $0.47$ for fixed
cycling; on Tiny-ImageNet, the intermediate schedules have slightly larger
point estimates than RPR. All bounded schedules have positive mean gains
over independent draws. Since renewal also changes the gap distribution,
these interventions do not independently identify the contribution of pair
diversity. The comparisons associate bounded gaps with positive mean accuracy
differences but do not isolate which temporal property accounts for them.

\subsection{Buffer capacity and class residency}

Table~\ref{tab:buffer-diagnostic} tests how memory capacity mediates the
scheduling effect on CIFAR-100 LT10. The accuracy columns report final
test accuracy, and $\Delta$ gives RPR minus independent balanced retrieval
with a $95\%$ paired interval. Deactivations count departure events from the
resident set, so a class can contribute more than once if it leaves, returns,
and leaves again. Reporting this count alongside the gap tail
distinguishes failure to retain a class from failure to revisit one that
remains available.

\begin{table}[htbp]
\centering
\caption{Accuracy (\%), class deactivations, and rehearsal gaps (steps)
across buffer capacities on CIFAR-100 LT10.}
\label{tab:buffer-diagnostic}
\small
\setlength{\tabcolsep}{3.5pt}
\resizebox{\linewidth}{!}{%
\begin{tabular}{rrrrrrr}
\toprule
Buffer & i.i.d. & RPR & $\Delta$ & Deactivations & i.i.d.\ p99/max & RPR p99/max \\
\midrule
$200$     & $8.40$  & $8.71$  & $+.30\ [+.01,+.59]$ & $16.1$ & $29.5/62.1$ & $14.7/18.1$ \\
$500$     & $9.89$  & $10.06$ & $+.17\ [-.32,+.65]$ & $5.9$ & $34.2/77.4$ & $17.9/21.8$ \\
$1{,}000$ & $10.77$ & $11.45$ & $+.68\ [+.13,+1.24]$ & $1.9$ & $36.6/78.0$ & $19.0/23.9$ \\
$2{,}000$ & $10.85$ & $12.01$ & $+1.16\ [+.64,+1.68]$ & $.3$ & $36.0/78.1$ & $19.4/24.3$ \\
$5{,}120$ & $10.97$ & $12.20$ & $+1.23\ [+.53,+1.92]$ & $0$ & $36.7/75.9$ & $19.5/24.1$ \\
\bottomrule
\end{tabular}}
\end{table}

RPR shortens the gap tail at every capacity, but the accuracy effect is
small at buffers $200$ and $500$, and the interval at $500$ includes zero.
The gains increase to $0.68$, $1.16$, and $1.23$ points at the three larger
capacities, where class deactivations fall from $1.9$ to zero. The effect
is not strictly monotone across the entire sweep. Residency and gap statistics
vary together in these comparisons, so their contributions are not isolated.
Retrieval cannot revisit a class while it is absent from memory.

\section{Loss interventions and cross-host comparisons}
\label{app:law}

Table~\ref{tab:law} examines how the learner's loss changes its response
to class-balanced retrieval. These contrasts compare the independent
balanced draw with uniform retrieval, not RPR with the balanced draw.
Mask removal tests dependence on ER-ACE's incoming-loss asymmetry;
mask addition tests whether the same intervention transfers to DER++.
The remaining rows probe the limits of a prediction based only on the
per-example form of the replay loss. This confirmation block contains
$470$ test runs. Entries report mean paired effects with $95\%$ Student-$t$
intervals; an effect is material when the interval excludes zero and its
relative magnitude exceeds the pre-specified $2\%$ threshold. CE denotes
per-example cross-entropy, whose normalizer ranges over output classes,
not other replay examples. Daggers identify null predictions specified
before execution.

\begin{table}[htbp]
\caption{Class-balanced minus uniform retrieval under loss interventions and across hosts.
Absolute effects are in percentage points.}
\label{tab:law}
\centering
\small
\setlength{\tabcolsep}{4.5pt}
\begin{tabular}{l l l l l c}
\toprule
test & host & replay term & \multicolumn{2}{l}{effect of \textsc{cb} draw, $95\%$ CI} & mat.\ \\
\midrule
\multirow{4}{*}{mask removal}
 & ER-ACE, mask on  & CE & $+2.26$ ($+9.9\%$) & $[+1.89,+2.63]$ & yes \\
 & ER-ACE, mask off & CE & $+0.00$ ($+0.0\%$) & $[-0.65,+0.65]$ & no \\
 & \emph{same, Tiny-IN}, on  & CE & $+2.03$ ($+13.9\%$) & $[+1.58,+2.49]$ & yes \\
 & \emph{same, Tiny-IN}, off & CE & $-0.10$ ($-0.9\%$) & $[-0.50,+0.31]$ & no \\
\midrule
\multirow{4}{*}{mask addition}
 & DER++, published    & per-example  & $-0.23$ ($-1.3\%$) & $[-0.64,+0.18]$ & no \\
 & DER++ $+$ ACE mask  & per-example & $+0.80$ ($+3.1\%$) & $[-0.04,+1.63]$ & no \\
 & \emph{same, Tiny-IN}, published & per-example & $-0.04$ ($-0.4\%$) & $[-0.33,+0.25]$ & no \\
 & \emph{same, Tiny-IN}, $+$ mask  & per-example & $+0.24$ ($+1.4\%$) & $[-0.21,+0.69]$ & no \\
\midrule
\multirow{4}{*}{prediction}
 & DER & logit regr. & $-0.02$ ($-0.2\%$) & $[-0.12,+0.08]$ & no$^\dagger$ \\
 & \emph{same, Tiny-IN}        & logit regr. & $-0.06$ ($-1.0\%$) & $[-0.14,+0.01]$ & no$^\dagger$ \\
 & DER++/Refresh               & unmasked CE & $+0.90$ ($+5.1\%$) & $[+0.38,+1.41]$ & yes$^\dagger$ \\
 & \emph{same, Tiny-IN}        & unmasked CE & $-0.24$ ($-2.2\%$) & $[-0.55,+0.07]$ & no$^\dagger$ \\
\bottomrule
\end{tabular}
\end{table}

Without the ER-ACE mask, the balanced-draw point estimates are near zero
and both intervals include zero. The mask itself improves ER-ACE by $5.39$ points
on CIFAR-100 and $4.16$ on Tiny-ImageNet, a separate contrast from the
retrieval effect in the table. Adding the mask to DER++ moves both point estimates
in the positive direction, but neither interval excludes zero. DER is
inconclusive, whereas DER++/Refresh on CIFAR-100 contradicts the predicted
null, with a $+0.90$-point effect. Together, these
controls support mask dependence in ER-ACE but not a necessary or
sufficient rule across hosts. They complement the RPR mask interaction in
Table~\ref{tab:mask-interaction}, which compares RPR with the independent
balanced draw at fixed storage and replay budget.

The scheduling and cross-host comparisons further qualify this account.
With an ACE mask, DER++ gains $1.02$ points on CIFAR-100 and $0.87$ on
Tiny-ImageNet from RPR over independent balanced retrieval; the corresponding
unmasked effects are $0.10$ and $0.24$ points. OCM has an inconclusive
$0.30$-point effect ($95\%$ CI $[-0.44,+1.05]$), rather than the negative
effect suggested by a simple batch-coupling argument. Here DER++/Refresh
denotes the unlearn-then-relearn perturbation with Fisher damping. Its
task-free adaptation updates Fisher estimates through an online moving
average triggered by observed-example count rather than task boundaries.
These results separate a local loss mechanism from the cumulative response
of an entire training trajectory.

\section{Accuracy across the learning trajectory}
\label{sec:trajectory}

To summarise accuracy during training, Table~\ref{tab:trajectory-auc} compares the arithmetic mean of
the ten task-boundary average-accuracy measurements within each pretrained
ViT-B/16 LT10 run. Entries give RPR minus the indicated comparator, with
$95\%$ Student-$t$ intervals over ten paired seeds. This metric summarises
the evaluated trajectory; it is not sample-wise online or prequential
accuracy and does not resolve changes between evaluation points.

\begin{table}[htbp]
\centering
\caption{RPR effects on mean task-boundary accuracy for pretrained ViT-B/16 on LT10 streams.
Differences are in percentage points.}
\label{tab:trajectory-auc}
\small
\setlength{\tabcolsep}{6pt}
\begin{tabular}{lcrr}
\toprule
Dataset & replay batch & RPR vs. i.i.d. & RPR vs. fixed cycle \\
\midrule
CIFAR-100 & $8$  & $+1.30\;[+0.59,+2.00]$ & $+0.05\;[-0.64,+0.74]$ \\
ImageNet-R & $8$ & $+1.13\;[+0.71,+1.54]$ & $+0.01\;[-0.67,+0.70]$ \\
CIFAR-100 & $32$ & $-0.14\;[-0.67,+0.38]$ & $+0.02\;[-0.35,+0.40]$ \\
ImageNet-R & $32$ & $+0.20\;[-0.27,+0.67]$ & $+0.39\;[-0.30,+1.09]$ \\
\bottomrule
\end{tabular}
\end{table}

At $b=8$, RPR improves this trajectory summary over independent balanced
retrieval by $1.30$ points on CIFAR-100 and $1.13$ points on ImageNet-R,
with both intervals above zero. At $b=32$, both intervals include zero,
consistent with the shorter resident-gap tail and weaker final-accuracy
effect in the matched replay-budget controls. Every fixed-cycle contrast
also includes zero. These measurements support higher mean task-boundary
accuracy for RPR than for independent draws in the tested small-batch setting,
but do not establish an additional effect over fixed cycling.

\section{Additional retrieval controls}
\label{app:additional}
\label{app:dissociation}

Table~\ref{tab:dissociation} compares four retrieval objectives in ER-ACE
with a $5{,}120$-example buffer and a replay batch of $32$. Entries report
test-accuracy differences from uniform retrieval in percentage points,
with relative changes in parentheses. MIR prioritises examples by
interference; the independent balanced draw samples classes uniformly; RPR
uses shuffled class passes. This comparison evaluates these retrieval
policies within the same ER-ACE configuration.

\begin{table}[htbp]
\caption{Test-accuracy effects of retrieval policies relative to uniform retrieval.
Absolute effects are in percentage points.}
\label{tab:dissociation}
\centering
\small
\begin{tabular}{l l rr}
\toprule
policy & retrieval rule & CIFAR-100 & Tiny-ImageNet \\
\midrule
MIR                     & interference-based selection  & $-16.38$ ($-50.1\%$) & $-8.18$ ($-35.1\%$) \\
uniform                 & uniform example sampling      & $\pm0$ & $\pm0$ \\
class-balanced          & independent balanced draw     & $+1.88$ ($+5.8\%$) & $+2.30$ ($+9.9\%$) \\
RPR                     & shuffled class passes         & $+2.43$ ($+7.4\%$) & $+2.91$ ($+12.5\%$) \\
\bottomrule
\end{tabular}
\end{table}

In these runs, MIR retrieves $7.4$ distinct classes per batch on average,
compared with uniform retrieval's $21.7$, and its accuracy is lower on both
datasets. Class-balanced retrieval and RPR instead improve over uniform
retrieval, with RPR giving the largest point estimate. On CIFAR-100, for
example, the relative changes are $-50.1\%$, $0\%$, $+5.8\%$,
and $+7.4\%$ for MIR, uniform retrieval, independent balanced retrieval,
and RPR, respectively. The association between class coverage and accuracy is
consistent with the ER-ACE mechanism, but this comparison changes the
retrieval objective as well as batch composition. It does not isolate a
causal effect of class count or establish a general ranking of MIR and RPR
across hosts. The separate readout control in
Appendix~\ref{sec:emarobust} tests whether the scheduling benefit depends on
evaluation-time weight averaging.

\end{document}